\documentclass[twoside]{article} % for aistats
\usepackage{etoolbox}
\usepackage{caption}

\newtoggle{conferenceVersion}

\toggletrue{conferenceVersion}   % for conference version
\iftoggle{conferenceVersion}{
     \usepackage[accepted]{aistats2024} % upon acceptance
}{
    \input{snippets/preprint_header}
    \usepackage{geometry}
    \usepackage{fullpage}
}

\usepackage{bm}
\usepackage{graphicx}
\usepackage{xcolor}         % colors
\usepackage{listings}

\definecolor{codegreen}{rgb}{0,0.6,0}
\definecolor{codegray}{rgb}{0.5,0.5,0.5}
\definecolor{codepurple}{rgb}{0.58,0,0.82}
\definecolor{backcolour}{rgb}{0.95,0.95,0.92}

\lstdefinestyle{mystyle}{
    backgroundcolor=\color{backcolour},
    commentstyle=\color{codegreen},
    keywordstyle=\color{magenta},
    numberstyle=\tiny\color{codegray},
    stringstyle=\color{codepurple},
    basicstyle=\ttfamily\footnotesize,
    breakatwhitespace=false,
    breaklines=true,
    captionpos=b,
    keepspaces=true,
    numbers=left,
    numbersep=5pt,
    showspaces=false,
    showstringspaces=false,
    showtabs=false,
    tabsize=2
}

\def\bX{\mathbf{X}}
\def\bZ{\mathbf{Z}}
\def\bG{\mathbf{G}}
\def\bM{\mathbf{M}}
\def\bN{\mathbf{N}}
\def\bP{\mathbf{P}}
\def\bI{\mathbf{I}}
\def\by{\mathbf{y}}
\def\ba{\mathbf{a}}
\def\bbeta{\bm{\beta}}
\def\bvarepsilon{\bm{\varepsilon}}
\def\bSigma{\bm{\Sigma}}
\def\Etest{\mathcal{E}_{\mathtt{test}}}
\def\Etrain{\mathcal{E}_{\mathtt{train}}}
\def\MNNI{\underline{\smash{\bbeta}}}
\def\RREG{\hat{\bbeta}}

\usepackage{amsmath,amssymb,amsthm}
\usepackage{hyperref}
\usepackage{cleveref}
\theoremstyle{plain}
\theoremstyle{definition}
\newtheorem{theorem}{Theorem}[section]
\newtheorem{proposition}[theorem]{Proposition}
\newtheorem{lemma}[theorem]{Lemma}

\newtheorem{definition}[theorem]{Definition}
\newtheorem{assumption}[theorem]{Assumption}
\newtheorem{example}[theorem]{Example}
\theoremstyle{remark}
\newtheorem{remark}[theorem]{Remark}

\crefname{theorem}{theorem}{theorems}
\Crefname{theorem}{Theorem}{Theorems}

\crefname{proposition}{Prop.}{pProp.}
\Crefname{proposition}{Proposition}{Propositions}

\crefname{lemma}{lemma}{lemmas}
\Crefname{lemma}{Lemma}{Lemmas}

\crefname{corollary}{corollary}{corollaries}
\Crefname{corollary}{Corollary}{Corollaries}

\crefname{definition}{definition}{definitions}
\Crefname{definition}{Definition}{Definitions}

\crefname{assumption}{assumption}{assumptions}
\Crefname{assumption}{Assumption}{Assumptions}

\crefname{example}{example}{examples}
\Crefname{example}{Example}{Examples}

\crefname{remark}{remark}{remarks}
\Crefname{remark}{Remark}{Remarks}

\usepackage[round]{natbib}

\begin{document}

\iftoggle{conferenceVersion}{
    % FOR CONFERENCE VERSION
    
\runningtitle{Near-Interpolators}
\runningauthor{Wang, Sonthalia, Hu}

\twocolumn[
\aistatstitle{Near-Interpolators: Rapid Norm Growth and\\the Trade-Off between Interpolation and Generalization}
\aistatsauthor{ Yutong Wang \And Rishi Sonthalia \And Wei Hu }
\aistatsaddress{ MIDAS, University of Michigan \And UCLA \And University of Michigan }
]

}{
    \maketitle
}

\begin{abstract}
  We study the generalization capability of nearly-interpolating linear regressors:  $\bm{\beta}$'s whose training error $\tau$  is  positive but small, i.e., below the noise floor.
Under a random matrix theoretic assumption on the data distribution
and an eigendecay assumption on
the data covariance matrix $\bm{\Sigma}$,
we demonstrate that any near-interpolator exhibits rapid norm growth: for $\tau$ fixed, $\bm{\beta}$ has squared $\ell_2$-norm $\mathbb{E}[\|{\bm{\beta}}\|_{2}^{2}] = \Omega(n^{\alpha})$ where $n$ is the number of samples
and $\alpha >1$ is the exponent of the eigendecay, i.e.,
$\lambda_i(\bm{\Sigma}) \sim i^{-\alpha}$.
This implies that existing data-independent norm-based bounds are necessarily loose.
On the other hand, in the same regime we precisely characterize  the asymptotic trade-off between interpolation and generalization.
Our characterization reveals that
larger norm scaling exponents $\alpha$ correspond to  worse trade-offs between interpolation and generalization.
We verify empirically that a similar phenomenon holds for nearly-interpolating shallow neural networks.

%   We study nearly-interpolating linear regressors:  $\bbeta$'s whose training error $\tau$  is  positive but small, i.e., below the noise floor.
% Under a random matrix theoretic assumption on the data distribution
% and an eigendecay assumption on
% the data covariance matrix $\bSigma$,
% we demonstrate
% that any near-interpolator exhibits rapid norm growth: for $\tau$ fixed, $\bbeta$
% has squared $\ell_2$-norm $\|{\bbeta}\|_{2}^{2} = \Omega(n^{\alpha})$ where $n$ is the number of samples
% and
%  $\alpha >1$ is the exponent of the eigendecay, i.e.,
% $\lambda_i(\bSigma) \sim i^{-\alpha}$.
% Moreover, we precisely characterize in this regime the trade-off between interpolation and generalization.
% A corollary of our characterization reveals:
% larger norm scaling exponents $\alpha$ correspond to  worse trade-offs between interpolation and generalization.
% We verify empirically that a similar phenomenon holds for nearly-interpolating shallow neural networks.
\end{abstract}

\section{INTRODUCTION}\label{section:introduction}
 Regularization \citep{nakkiran2020optimal} and early stopping \citep{ji2021early} are techniques to mitigate the effect of harmful overfitting
by training models to \emph{nearly}, rather than perfectly, interpolate the training data.
A key question is: \emph{how do near-interpolators generalize?}

A long line of work has investigated this question for \emph{perfect}-interpolators.
\cite{zhang2017understanding} noted the surprising phenomon that, even with noise,
perfect-interpolators do not necessarily overfit catastrophically, and can still generalize to some extent.
The phenomenon is later formalized as ``benign overfitting'' and proven to hold in linear regresion in \cite{bartlett2020benign}.
% Later, \cite{nagarajan2019uniform} show that there exist practical settings under which uniform generalization bound cannot explain this surprising fact.
\cite{mallinar2022benign} introduced the more nuanced notion of \emph{tempered} overfitting
which is closer to the empirical observation in \cite{zhang2017understanding} that the test error of perfect-interpolators do degrade somewhat.
\cite{koehler2021uniform} establish a setting under which benign overfitting \emph{can} be explained by uniform bounds.
However, to the best of our knowledge,
no work have studied generalization of near-interpolating linear regression
\footnote{
\cite{ghosh2022universal} establishes
  \emph{lower bound} for the interpolation-genearlization trade-off, which is related to but distinct from our contributions.
}.

\textbf{Our results}. Under random matrix-theoretic and power-law spectra assumptions, we
prove that
nearly-interpolating ridge regressors have norms that grows rapidly
(\Cref{theorem:polynomial-lower-bound-specialized}),
implying that existing (non-asymptotic) generalization bounds are loose
(\Cref{section:looseness}).
Moreover, we
derive the exact formula relating the large-sample limit training and the testing error
(\Cref{theorem:trade-off}), using the eigenlearning framework of \cite{simon2021neural}.
Finally, we show that our theoretical results on near-interpolating ridge regressors are relevant empirically and can give insight into the behavior of early-stopped near-interpolating shallow neural networks.

% \textbf{Related works}.

\textbf{Implications}.
% \textcolor{red}{Start with the norm growth part, then follow-up with the discussion about interpolation-generalization.}
Our result on the norm growth 
implies that
existing data-independent bounds and possible extension are necessarily loose. See \Cref{section:looseness}.
Thus, in order to explain the learning capability of near-interpolators,
there is a need to develop data-dependent generalization bounds.

On the other hand, our result allows the analysis of the trade-off between nearness of interpolation and generalization.
Our result reveals delicate interplay between the overparametrization ratio and the power-law spectra exponent.
In particular, for larger  power-law spectra exponent implies larger asymptotic excess test error \emph{ratio} of  \(5\%\)- over \(50\%\)-noise floor interpolation, for instance. Put more simply, the harmfulness of overfitting depends on the data distribution. Moreover, this effect is stronger at high level of overparametrization (large \(p/n\)). (See \Cref{figure:steepness} and \Cref{fig:experiment-norms}-left panel.)
Experimentally, this effect appears in shallow neural networks as well (\Cref{figure:experiment-NN}).

\subsection{Related works}

\textbf{Near interpolation}. Learning algorithms that (nearly) interpolate the training data, such as deep neural networks, have been surprisingly effective in practice despite conventional statistical wisdom suggesting otherwise  \citep{zhang2021understanding}.
Many practices in modern machine learning e.g., early stopped neural network and high-dimensional ridge regression, result in
\emph{near-} rather than \emph{perfect-}interpolators \citep{ji2021early,kuzborskij2022learning}. In terms of theory work, \citet{ghosh2022universal} provides a \emph{lower bounds} on the {test error} for near-interpolators.

% \rishi{Norm-based bounds}

% To quote
% Mikhail Belkin:
% ``there are no [classical uniform] bounds [that explains generalization of interpolators] and no reason they
% should exist.'' (We learned of this quote from \cite{zhou2020uniform}).

\textbf{Power law spectra}.
Empirically, power law spectra  arise in neural tangent kernels computed from practical networks for common datasets, e.g.,
\textsc{MNIST} \citep{velikanov2021explicit}
% , \textsc{CIFAR-100} \citep{wei2022more} and \textsc{Caltech 101} \citep{murray2022characterizing}.
On the theory side, the power law spectra assumption has been used previously to analyze \emph{benign}  \citep[Theorem 2]{bartlett2020benign} and
\emph{tempered overfitting} phenomena \citep[Theorem 3.1]{mallinar2022benign}.

\textbf{Looseness of existing generalization bounds}. Our work is motivated by the empirical evidence found by \citet{wei2022more} suggests that norms of kernel ridge regressors grow rapidly potentially beyond the purview of norm-based bound. We confirm that bounds similar to
ones in \citet[Corollary 1]{koehler2021uniform} grow to infinity. Even more refined bounds such as \citet[Theorem 1]{koehler2021uniform} grow as $n$ goes to infinity.
Therefore, our work suggests that explaining the generalization capability of near-interpolators will require new tools.

% Our result shows that such near-interpolators have norms increase super-linearly in the number of samples
% and exhibit tempered \emph{near}-overfitting.
% Consequently, current norm-based generalization bounds are not applicable to explained this tempered near-overfitting behavior, and that tighter bounds are needed in the power law spectra assumption.

% The main difference between our work and that of  \citet{mallinar2022benign} is that our work establishes super-linear growth of the squared norm of near-interpolators. 

\subsection{Notations}\label{section:notations}
Throughout this work, we assume the setting of high-dimensional linear regression as described below.
Let
\(n\) denote the number of samples and \(p\) denote the feature dimension.
Consider the setting where \(n,p \to \infty\) at the same time.
The \emph{sample-to-feature ratio} is denoted
\(\gamma := n/p \in \mathbb{R}_{>0}\) and the asymptotic sample-to-feature ratio is denoted \(\gamma_{\ast} := \lim_{n\to\infty} \gamma \in \mathbb{R}_{\ge 0}\).
Here, \(n\) is the fundamental parameter which \(p\) depends on implicitly.

Let \(X \in \mathbb{R}^p\) and \(Y \in \mathbb{R}\) denote a random vector (resp.\ variable), referred to as the \emph{sample} (resp.\ \emph{label}).
Suppose that \(\bbeta^\star \in \mathbb{R}^p\)  is such that
\(Y = \varepsilon + X^\top \bbeta^\star\)
where
\(\varepsilon \in \mathbb{R}\) is a random variable denoting independent, zero mean noise, i.e., \(\mathbb{E}[\varepsilon] = 0\) and \(\varepsilon \perp X\). Here \(\perp\) denotes independence between random variables.
The noise variance is denoted
$\sigma^2 := \mathbb{E}[\varepsilon ^2]$.

The training data is denoted \(\{(x_i,y_i)\}_{i=1}^n\) where \(x_i \in \mathbb{R}^p\) and \(\varepsilon_i \in \mathbb{R}\) are i.i.d realizations of \(X\) and of the noise, and
\(y_i = \varepsilon_i + x_i^\top \bbeta^\star\).
Let
\(\bX = [x_{1},\dots, x_{n}] \in \mathbb{R}^{p \times n}\) be the data matrix obtained by horizontally stacking the \(x_i\)'s,
and
let $\by = (y_1,\dots, y_n)^\top \in \mathbb{R}^n$ be the (column vector) by concatenating the \(y_i\)'s.
Likewise, define \(\bvarepsilon = (\varepsilon_1,\dots, \varepsilon_n)^\top \in \mathbb{R}^n\).
For a positive integer \(p\), let
$\bI_{p}$ denote the \(p\times p\) identity matrix.
Let \(\bbeta \in \mathbb{R}^p\) be arbitrary. The \emph{empirical} training  mean squared  error (MSE) of \(\bbeta\) is denoted
\[
\Etrain^{n}(\bbeta) = \frac{1}{n} \| \bX^{\top} \bbeta - \by \|_{2}^{2}.
\]
Likewise, the \emph{expected} test error of \(\bbeta\) is denoted
\[
\Etest^{n}(\bbeta) = \mathbb{E}[ \| X^{\top} \bbeta - Y \|_{2}^{2}].
\]

Let \(\hat{\bSigma} := n^{-1}\bX\bX^\top\) denote the sample covariance matrix and
\(\check{\bG} := n^{-1}\bX^\top \bX\) the (scaled) gram matrix.
Let
\(\bSigma := \mathbb{E} [ \hat{\bSigma}]\) denote the population covariance.

We note that all quantities defined on the training data implicitly depend on \(n\).
When the dependencies need to be made explicit, we shall write
\(\bbeta^{\star}_{n}\),
\(\hat{\bSigma}_{n}\) and so on.

\subsection{Our contributions} \label{section:setup}

Recall the minimum norm near-interpolator:

\begin{definition}\label{definition:constrained-RR}
  Let \(\tau \in (0, \sigma^{2})\).
  % arbitrary\footnote{
  %   It is possible to define a (minimum norm) \(\tau\)-near-interpolator for any \(\tau \in \mathbb{R}_{\ge 0}\). However, the terminology ``interpolator'' is only appropriate when we are fitting the training data below the noise floor.
  % }.
  The \emph{minimum norm \(\tau\)-near-interpolator} is defined as
\begin{equation}
\MNNI_{\tau}:=
\mathrm{argmin}_{\bbeta \in \mathbb{R}^p}  \| \bbeta \|_2^2
\,\, s.t. \,\, \tfrac{1}{n} \| \bX^{\top} \bbeta - \by\|_{2}^{2} \le \tau.
\label{equation:constrained-formulation}
\end{equation}
A \emph{\(\tau\)-near-interpolator} (not necessarily of the minimum norm) is any \(\bbeta \in \mathbb{R}^p\) satisfying the inequality in \eqref{equation:constrained-formulation}.
\end{definition}

In the overparameterized (\(p > n\)) regime, near-interpolators can often be realized by ridge regression:

\begin{definition}\label{definition:RR}
Let \(\varrho >0\). The \emph{ridge regressor with
regularizer}
$\varrho$
is
 defined as
\begin{equation}
\hat{\bbeta}_\varrho :=\mathrm{argmin}_{\bbeta \in \mathbb{R}^p} \tfrac{1}{n} \| \bX^\top \bbeta - \by\|_2^2 + \varrho \| \bbeta \|_2^2.
\label{equation:KRR}
\end{equation}
\end{definition}

\textbf{Main problems}: For any \(\tau \in (0,\sigma^{2})\), \underline{\emph{1}}.\ find
a sequence of regularizers \(\{\varrho_{n}\}_{n}\) so that
\begin{equation}
\Etrain^{\ast}:=
\lim_{n\to \infty} 
\mathbb{E}[\Etrain^{n}(\RREG_{\varrho_{n}})] = \tau
\label{equation:asymptotic-training-identity}
\end{equation}
where the expectation is over all sources of randomness
and \underline{\emph{2}}.\ compute the associated asymptotic test error:
\begin{equation}
\Etest^{\ast}:=
\lim_{n\to \infty} \Etest^{n}(\RREG_{\varrho_{n}}).
\label{equation:asymptotic-testing-identity}
\end{equation}
\begin{definition}
Let \(\tau \in (0,\sigma^{2})\) be arbitrary and \(\{\varrho_{n}\}_{n}\) be any sequence of regularizers.
If \Cref{equation:asymptotic-training-identity} holds, then we say the sequence of regressors
\(\{\RREG_{\varrho_{n}}\}_{n}\) is an \emph{asymptotic \(\tau\)-near interpolator}
with asymptotic test error given by \Cref{equation:asymptotic-testing-identity}.
\end{definition}
\begin{assumption}[Power-law spectra\footnote{Also referred to as the {eigenvalue decay} condition \citep{goel2017eigenvalue}.}]\label{assumption:exact-EVD}
  Suppose there exists
   \(\alpha > 1\) such that the population data covariance matrix
  \(    \bSigma
    =
    \mathrm{diag}(\lambda_{1},\cdots, \lambda_{p})
\)
  where \(\lambda_{i} = i^{-\alpha}\).
\end{assumption}
There are many examples of random matrix ensembles exhibiting power-law spectra in a broader sense than that of  \Cref{assumption:exact-EVD}.
For instance, see \citep{arous2008spectrum,mahoney2019traditional,wang2024spectral}. For simplicity, we do not pursue a general setting and will consider the setting of  \Cref{assumption:exact-EVD}.

The asymptotic test error of an asymptotic \(\tau\)-near interpolator can be calculated as follows.
  Let \(F = {}_{2}F_{1}\) denote the Gaussian Hypergeometric function \cite[Eqn.\ (27)]{dutka1984early}.

\begin{figure*}
  \centering
\includegraphics[width = 0.44\textwidth]{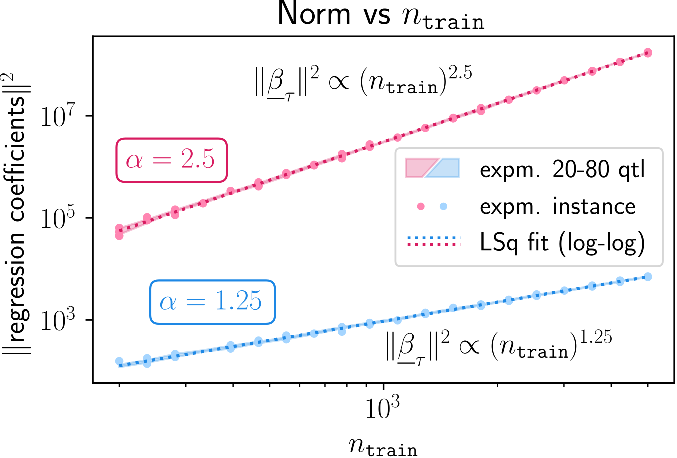}
~
~
~
\includegraphics[width = 0.42\textwidth]{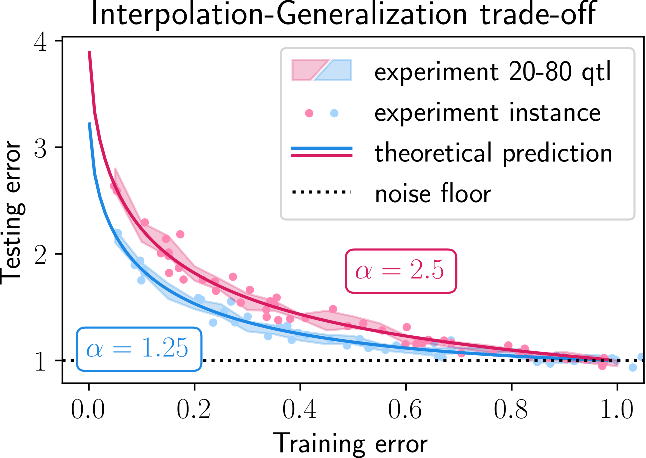}
\caption{
\label{fig:experiment-norms}
\textbf{Left}: \emph{Synthetic experiments validating the
    norm lower bound of norms of \(0.2\)-near-interpolators given by \Cref{theorem:polynomial-lower-bound}}.
The squared norms are fitted by least squares (in log-log space)
to estimate the norm-growth exponent \(\alpha\) using only data points.
See \Cref{section:experiment} for additional experiment details.
\textbf{Right}:
    \emph{Trade-off between the testing and training errors from \Cref{theorem:trade-off}.}
    The solid curves are the parametrized curves whose \((x,y)\)-coordinates are
    \((
\mathcal{E}_{\mathtt{train}}^{\ast},
    \mathcal{E}_{\mathtt{test}}^{\ast}
      )\) and parametrized by \(k\) (which is in 1-to-1 correspondence with \(r\) see Theorem \ref{theorem:trade-off}).
    The scatter points, subsampled for visualization, denote   ridge regression run   results on the HDA model (\Cref{example:HDA}).
    The colored ribbons denote the 20-80 quantiles for the scatter points.
    The horizontal dotted line denotes the noise \(\sigma^2\) which is set to \(1\) without the loss of generality.
}
\end{figure*}
\begin{theorem}[Exact trade-off formula]\label{theorem:trade-off}
  Let \(\tau \in (0,\sigma^{2})\) be arbitrary.
  Suppose that \(\sup_{n=1,2\dots}\|\bbeta^{\star}\|_{2} < +\infty\),
\Cref{assumption:exact-EVD} holds, and
  \(X = \bSigma^{1/2} Z\) where \(Z \sim \mathcal{N}(0,\bI_{p})\).
  There exists unique
number
\(k \in \mathbb{R}_{>0}\) such that
that the following hold.
Define the \emph{regularizer-factor} \[r :=
      1 -
\gamma_{\ast}^{-1} F(1, \tfrac{1}{\alpha}; 1 + \tfrac{1}{\alpha}; -k \gamma_{\ast}^{-\alpha})
\]
and let
\(
\varrho_{n} := r n^{-\alpha}\).
  Then
\(\{\RREG_{\varrho_{n}}\}_{n}\) is an {asymptotic \(\tau\)-near interpolator}
whose
asymptotic test error is
\begin{equation}
\Etest^{\ast}
  =
  \sigma^{2}
\frac{1}
{
  1 -
\gamma_{\ast}^{-1} F(2, \tfrac{1}{\alpha}; 1 + \tfrac{1}{\alpha}; -k \gamma_{\ast}^{-\alpha})
  }.
\label{equation:exact-trade-off-formula}
\end{equation}
Moreover, \(\mathcal{E}^{\ast}_{\mathtt{test}}\) is a decreasing function w.r.t \(\alpha\), fixing all other quantities.
\end{theorem}

The reason we call \Cref{theorem:trade-off} an ``exact trade-off formula'' is that
\Cref{equation:exact-trade-off-formula} allows the calculate of the trade-off curve
between train and test error
(\Cref{fig:experiment-norms}-right).
 The fundamental parameter is \(k\). The asymptotic testing 
error and training error, i.e., \(\tau\), all depend on \(k\) via monotonic 1-1 correspondences on the domain \(k \in (k_{\mathtt{crit}}, \infty)\).
See \Cref{figure:R-func} below. Thus, the asymptotic testing error depends on the training error implicitly through \(k\).

\Cref{fig:experiment-norms}-right panel demonstrates that, empirically, training and test MSEs concentrate closely around \Cref{equation:exact-trade-off-formula}.
We further discuss in detail the implications of 
 \Cref{theorem:trade-off} after  stating \Cref{proposition:trade-off}.

\Cref{fig:experiment-norms}-right shows that for near-interpolators, the test error does not degrade much when training below the noise floor.
 A natural question is if this can be explained by data-independent norm-based generalization bound such as the one found in \cite{koehler2021uniform}.
Our next result shows that the growth rate of an
asymptotic \(\tau\)-near interpolator is superlinear:
\begin{theorem}[Rapid norm growth]\label{theorem:polynomial-lower-bound-specialized}
  In the situation of \Cref{theorem:trade-off}, for any \(\tau \in (0,\sigma^{2})\), suppose \(\varrho_{n}>0\) is a sequence of regularizers
such that
\(\{\RREG_{\varrho_{n}}\}_{n}\) is an {asymptotic \(\tau\)-near interpolator}.
Then
\(\mathbb{E}[\|\hat{\bbeta}_{\varrho_{n}}\|_{2}^{2}] = \Omega(n^{\alpha})\).
\end{theorem}
As a consequence, data-independent norm-based generalization bound for near-interpolators, similar to the one in \cite{koehler2021uniform}, are necessarily loose.
See \Cref{section:looseness}.

The key technical result that enables the proof of
\Cref{theorem:polynomial-lower-bound-specialized}
is the following

\begin{proposition}[Rapid norm growth - generic]\label{theorem:polynomial-lower-bound}
  Suppose \Cref{assumption:exact-EVD} holds and the random matrix-theoretic Assumptions \ref{assumption:self-consistent-kr},
  \ref{assumption:MP-law}
  and
  \ref{assumption:positivity-condition} all hold.
  % Assume that
  % the exact\footnote{
  %   This assumption can be replaced with the weaker \Cref{assumption:EVD}, which we prove in the Appendix.
  % } EVD
  % (\Cref{assumption:exact-EVD})
  % and
  % certain random matrix-theoretic\footnote{
  % The \emph{positivity condition} to be introduced in
  % \Cref{assumption:positivity-condition}.
  % }
  % conditions hold.
  For any \(r > 0\), let \(\varrho_{n} := r n^{-\alpha}\) be the regularizer for the ridge regression.
Then
\(\mathbb{E}[\|\hat{\bbeta}_{\varrho_{n}}\|_{2}^{2}] = \Omega(n^{\alpha})\).
\end{proposition}

\begin{remark}
  \Cref{theorem:polynomial-lower-bound}
  still holds when the stronger 
  \Cref{assumption:exact-EVD}
  is replaced by the
   weaker \Cref{assumption:EVD}. See \Cref{theorem:polynomial-lower-bound-2}.
\end{remark}

\begin{remark}[Effective-factor]\label{remark:effective-factor}
  The quantity k in \Cref{theorem:trade-off} has the following interpretation.
Let \(\kappa := k n^{-\alpha}\), which is
known as the \emph{effective regularizer} in \citet{wei2022more}.
The connection between the effective regularizer and the  statistical learning theoretic-literature's notion of \emph{effective dimension} is explained in \citep[\S 4.1]{jacot2020implicit}.
For this reason, we refer to \(k\) with the shortened name \emph{eff-reg-factor}.

\end{remark}

\subsection{Organization}
In \Cref{section:RMT-primer}, we present the necessary background as well as new technical on random matrix theory (RMT). 
In \Cref{section:contrib-errors}
and
\ref{section:rapid-norm-growth}, we sketch the proof of \Cref{theorem:trade-off}
and
\Cref{theorem:polynomial-lower-bound}, respectively.
In
\Cref{section:looseness}, we discuss the implication of our results on the looseness of norm-based generalization bounds.
In \Cref{section:experiment}, we discuss our experiments.
We discuss related works and the context of our work in greater details in \Cref{section:more-related-works}. Finally, we conclude with discussion of future works and limitations.
% \[
%   \Etrain^{\ast}
%   =
%   \sigma^{2}\cdot
%   \frac{
%     \left(
%       1 -
% \gamma_{\ast}^{-1} F(1, \tfrac{1}{\alpha}; 1 + \tfrac{1}{\alpha}; -k \gamma_{\ast}^{-\alpha})
% \right)^{2}
%   }
%   {
%   1 -
% \gamma_{\ast}^{-1} F(2, \tfrac{1}{\alpha}; 1 + \tfrac{1}{\alpha}; -k \gamma_{\ast}^{-\alpha})
%   }.
% \]

% \subsection{Implications of our contributions}

% \begin{remark}
%   Later, we assume that the regularizer \(\varrho := rn^{-\alpha}\)
%   where
%   \(\alpha\) is the exponential of the eigenvalue decay of \(\bSigma\).
% \end{remark}

% \begin{figure*}
%   \centering
% \includegraphics[width = 0.44\textwidth]{figures/nn2_trade_off.eps}
% ~
% ~
% \includegraphics[width = 0.45\textwidth]{figures/nn2_norm_growth.eps}
% \caption{Hello world}
% \end{figure*}

\section{PRIMER ON RANDOM MATRIX THEORY}\label{section:RMT-primer}
We start with a fundamental concept in random matrix theory (RMT), followed by a review of RMT adapted to the power-law spectra setting and a new result (\Cref{theorem:MP}) to prepare for our main results.

% \textbf{Covariance/gram matrices}.

\begin{definition}[Empirical spectral measure]
For $c \in \mathbb{R}$, let $\delta_c$ denote the {Dirac measure} on $\mathbb{R}$ at $c$.
  Let $\bM \in \mathbb{R}^{p \times p}$ be a matrix with real eigenvalues $\lambda_1,\dots, \lambda_p$.
  The \emph{empirical spectral measure} of $\bM$, denoted by $\mathtt{esd}(\bM)$, is the measure on $\mathbb{R}$ given by
$\mathtt{esd}(\bM) = \frac{1}{p}\sum_{i=1}^p \delta_{\lambda_i}$.
\end{definition}
Our random matrix theoretic-assumptions differs from the standard RMT ones in order to accomodate for power-law spectra.
The following is a random matrix-theoretic extension of
the earlier \Cref{assumption:exact-EVD}:

\begin{assumption}[Power-law spectra, RMT version]\label{assumption:EVD}
In the situation of \Cref{section:notations},
  let \(\alpha >1\) and \(H\) be some probability measure on \(\mathbb{R}_{\ge 0}\).
Assume that \(\mathtt{esd}(n^{\alpha}\bSigma)\) converges to \(H\) (in the sense of convergence in distribution).
We refer to \(H\)  as the \emph{\(\alpha\)-scaled limiting spectral distribution (\(\alpha\)-scaled 
 LSD)}.
\end{assumption}
Morally, we can think of the above \(\alpha\) as the same as that of \Cref{assumption:exact-EVD}.

\begin{remark}[Comparison with standard LSD]\label{remark:standard-LSD}
In RMT, the condition that  ``\(\mathtt{esd}(\bSigma)\) converges to \(H\)'' is standard, where \(H\) is simply referred to as the \emph{limiting spectral distribution} (LSD) \citep{bai2010spectral}.
For power-law spectra covariance, i.e., \(\bSigma\) satisfying \Cref{assumption:exact-EVD},  \(\mathtt{esd}(\bSigma)\)  may not have a measure-theoretic limit while
 \(\mathtt{esd}(n^{\alpha}\bSigma)\) does, as we will show in \Cref{section:limitingd-istribution}.
 \end{remark}
 
\begin{definition}[Stieltjes transform]
  Let $\mu$ be a measure on $\mathbb{R}$ with support $S$. The \emph{Stieltjes transform}
of $\mu$ is the (complex-valued) function with input \(z \in \mathbb{C}\setminus S\) given by
$
\mathcal{S}_\mu(z) := \int \frac{ \mu(t) dt}{t - z}$.
\end{definition}

% \begin{proposition}\label{proposition:HDA-satisfies-positivity-condition}
%   Under the HDA model (\Cref{example:HDA}),
%   the positivity condition
%   (\Cref{assumption:positivity-condition}) holds.
% \end{proposition}
Next, we recall the so-called the \emph{self-consistent equation} \citep{tao2011stieltjes} which relates the regularizer-factor \(r\) with the eff-reg-factor \(k\):
\begin{assumption}[Self-consistent equation]\label{assumption:self-consistent-kr}
In the situation of \Cref{section:notations}, denote by \(\lambda_i\) the \(i\)-th largest eigenvalue of \(\bSigma\).  For each \(r > 0\), there exists a unique \(k \equiv k(r) \in \mathbb{R}\) such that
  the tuple \((r,k)\) satisfies
  \begin{equation}
    1 = \frac{r}{k}  +
    \lim_{n\to\infty} \frac{1}{n} \sum_{i=1}^{p}
    \frac{
      1
    }{1 + kn^{-\alpha}\lambda_{i}^{-1}}.
\label{equation:self-consistent-kr}
\end{equation}
\end{assumption}
Similar to 
\Cref{remark:standard-LSD}, \Cref{equation:self-consistent-kr} includes a \(n^{-\alpha}\) scaling term that does not show up in the standard self-consistent equation.
Again, our \Cref{assumption:self-consistent-kr} differs from this standard one in order to deal with the power-law spectra.

Next, we state a version of the classical Marchenko-Pastur law for a random matrix \(\bX\) (and its associated Gram matrix \(n^{\alpha} \check{\bG}_{n}\)):
% See
% \citep[Appendix]{wei2022more} for a discussion.

\begin{assumption}[Marchenko-Pastur law]\label{assumption:MP-law}
  In the setting of \Cref{assumption:self-consistent-kr},
    further assume that 
  \[
    \lim_{n\to\infty}  r\mathcal{S}_{\mathtt{esd}(n^{\alpha} \check{\bG}_{n})}(-r) =
    k
   \mathcal{S}_{H} (-k)
   , \quad \mbox{almost surely}
  \]
  and
  \(    \lim_{n\to\infty} \tfrac{d}{dr}\left(r \mathcal{S}_{\mathtt{esd}(n^{\alpha} \check{\bG}_{n})}(-r)\right) =
\tfrac{d}{dr}   \left(k \mathcal{S}_{H} (-k)\right)
\). We note that the \(k\) on the RHS
  depends on \(r\).
\end{assumption}

\begin{remark}\label{remark:MP}
\Cref{assumption:self-consistent-kr}
  and
  \Cref{assumption:MP-law}
  are standard assumptions in random matrix theory. Both of them are satisfied by
     the well-studied \emph{high-dimensional asymptotic} (HDA)  model (\Cref{example:HDA}).
  For instance, see \citet{dobriban2018high} under ``Marchenko-Pastur theorem''.
\end{remark}

The HDA model serves as an exemplary model in  random matrix theory possessing  many properties  that are particularly amenable to analysis. 
It is defined  as:

\begin{example}
  \label{example:HDA}
  Let \(\gamma_\ast \in (0,\infty)\).
  The \emph{high-dimensional asymptotic (HDA)\footnote{See \cite{bai2010spectral,dobriban2018high}.}} model:
  \begin{enumerate}
    \item
          \(\bX = \bSigma^{1/2} \bZ\) where the entries of \(\bZ = \{Z_{ij}\} \in \mathbb{R}^{p\times n}\) are i.i.d,
          have zero mean \(\mathbb{E}[Z_{ij}] = 0\) and
          unit variance \(\mathbb{E}[Z_{ij}^{2}]= 1\). The matrix \(\bSigma\) is positive semidefinite.
    \item
          \(n/p \to \gamma_{\ast} \)

    \item
          Spectral distribution of \(n^{\alpha}\bSigma\) converges to a distribution \(H\) supported on \(\mathbb{R}_{\ge 0}\).
  \end{enumerate}

\end{example}

Note that our   \Cref{example:HDA} is somewhat different compared to the \emph{conventional} HDA model, wherein the third item is
``\(\bSigma\) converges to a distribution \(H\)''.
Since we are working with power-law spectra in the covariance matrix, we require the \(n^\alpha\) coefficient in our \Cref{example:HDA}.

We now state a new random matrix-theoretic assumption that is one of the key steps for proving  rapid norm growth under the RMT setting (\Cref{theorem:polynomial-lower-bound-2}):
\begin{assumption}[Positivity condition]\label{assumption:positivity-condition}
  In the setting of \Cref{assumption:self-consistent-kr},
    further assume that 
  for every \(r > 0\),
  we have
  \[  \lim_{n\to\infty}
  \mathbb{E}
\left[
\frac{d }{d r}(
r \mathcal{S}_{\mathtt{esd}(n^{\alpha} \check{\bG}_{n})}(-r)
)
\right]
>0.\]
\end{assumption}

We show that the HDA model satisfies
\Cref{assumption:positivity-condition}, a fact that
appears to be new:
\begin{proposition}\label{theorem:MP}
\Cref{assumption:positivity-condition} holds for the HDA model.
\end{proposition}
We prove the proposition in \Cref{section:positivity}.
Now, having introduced the necessary RMT background, we now turn to proving \Cref{theorem:trade-off} on the interpolation-generalization trade-off.

\section{INTERPOLATION-GENERALIZATION TRADE-OFF}\label{section:contrib-errors}

\citet{simon2021neural} derived ``estimates'' of the testing and training errors of kernel ridge regression.
These estimates, dubbed the \emph{eigenlearning framework}, are non-rigorous\footnote{See \citet{mallinar2022benign} for a thorough discussion. Works in similar vein 
include
\citet{bordelon2020spectrum,canatar2021spectral}} due to invoking a \emph{Gaussian universality} condition.
However, when the kernel is linear and the data {is} Gaussian (as is the case in \Cref{theorem:trade-off}), the framework \emph{is} rigorous. See \citet{jacot2020kernel}.

Given this, we use the eigenlearning framework to rigorously calculate the asymptotic training and testing error of the estimators in \Cref{theorem:trade-off}. To this end, we first define two key functions of the {eff-reg-factor} \(k\) (See \Cref{remark:effective-factor} for the terminology):

\begin{definition}\label{definition:IJ-functions}
  Let \(\alpha > 1\) and \(\gamma_{\ast} \in [0,\infty)\). Define
  functions\footnote{Throughout this work, \(\alpha\) and \(\gamma_{\ast}\) are fixed constants.
For brevity, we often simply write \(\mathcal{I}\) or \(\mathcal{J}\).}
  \(\mathcal{I}(\cdot)
  \equiv
\mathcal{I}_{\alpha,\gamma_{\ast}}(\cdot)
  \)
  and
  \(\mathcal{J}(\cdot)
  \equiv
\mathcal{J}_{\alpha,\gamma_{\ast}}(\cdot)
  \)
  as
\[
  \mathcal{I}(k)
:=
    \int_{0}^{\tfrac{1}{\gamma_{\ast}}}
    % (1+k{x}^{\alpha})^{-1}dx,
    \frac{dx}{1+ k {x}^{\alpha} },
    \quad
       \mathcal{J}(k)
:=
    \int_{0}^{\tfrac{1}{\gamma_{\ast}}}
    \frac{dx}{(1+ k {x}^{\alpha} )^{2}}.
\]
When \(\gamma_{\ast} = 0\), we take \(1/0 := +\infty\).
\end{definition}
These integrals arise in explicit calculations of  the eigenlearning equations for the train and test MSEs applied to our setting. They can be computed via the integral representation of the Gaussian hypergeometric function given in  \citep[Eqn.\ (27)]{dutka1984early}.
The calculations are in \Cref{section:closed-form-appendix}, where we show that \(\mathcal{I}\) and \(\mathcal{J}\) are, respectively, equal to
\[
\int_{0}^{\tfrac{1}{\gamma_{\ast}}}
    % (1+k{x}^{\alpha})^{-1}dx,
    \frac{dx}{1+ k {x}^{\alpha} }=
\gamma_{\ast}^{-1} F(1, \tfrac{1}{\alpha}; 1 + \tfrac{1}{\alpha}; -k \gamma_{\ast}^{-\alpha}),\,\, \mbox{and}
\]
\[
    \int_{0}^{\tfrac{1}{\gamma_{\ast}}}
    \frac{dx}{(1+ k {x}^{\alpha} )^{2}}=
\gamma_{\ast}^{-1} F(2, \tfrac{1}{\alpha}; 1 + \tfrac{1}{\alpha}; -k \gamma_{\ast}^{-\alpha}).
\]
To relate the above to
\(\mathcal{E}^{\ast}_{\mathtt{test}},
  \Etrain^{\ast}\)
  estimates of the testing and training errors in the eigenlearning framework, we prove
\begin{proposition}\label{proposition:trade-off}
In the situation of \Cref{theorem:trade-off},
\[
  \mathcal{E}^{\ast}_{\mathtt{test}} \equiv
  \lim_{n\to\infty} \Etest^n(\hat{\bbeta_{\varrho}}) =
\sigma^{2}\cdot
\tfrac{1}
{
  1 - \mathcal{J}(k)
  }, \quad \mbox{and}
 \] \[
  \Etrain^{\ast}
  \equiv
  \lim_{n\to\infty}
  \mathbb{E}[\Etrain^n(\hat{\bbeta_{\varrho}})]
  =
  \sigma^{2}\cdot
  \tfrac{(1-\mathcal{I}(k))^{2}}
  {
  1 - \mathcal{J}(k)
  }.
\]
Moreover, there exists  \(k_{\mathtt{crit}} \in \mathbb{R}_{\ge 0}\) such that
\begin{enumerate}
  \item
For each \(r >0\), there exists a unique \(k \in (k_{\mathtt{crit}},+\infty)\) such that
\(r = \mathcal{R}(k):= k (1- \mathcal{I}(k))\),
\item \(\mathcal{R} \) is monotonically increasing on \((k_{\mathtt{crit}}, +\infty)\),
  \item \(\Etest^{\ast} > \sigma^{2}\) for all \(k \in (k_{\mathtt{crit}},+\infty)\),
  \item
\(\lim_{k\to+\infty}\Etest^{\ast} = \sigma^{2}\).
% \item 
% \(\frac{d}{d\alpha} \Etest^{\ast}>0\).
\end{enumerate}
  \end{proposition}
For the proof of \Cref{proposition:trade-off}, see \Cref{section:exact-evd}. At a high level, to prove the first part we apply the eigenlearning framework while accounting for the additional layer of complexity due to the power-law spectra.
For the ``Moreover'' part, we directly analyze \(\mathcal{R}\), \(\Etrain^{\ast}\) and \(\Etest^{\ast}\) as functions of \(r\), \(k\) and \(\alpha\).
Now, note that \Cref{proposition:trade-off} immediately implies \Cref{theorem:trade-off}.  

We now discuss some of the consequences of \Cref{proposition:trade-off}
  First,  \(\mathcal{R}\) is a bijection that relates the eff-reg-factor \(k\) and the regularizer-factor \(r\). The plot of \(\mathcal{R}\) is visualized in \Cref{figure:R-func}.
  Furthermore, note that \(\lim_{k\to+\infty}\Etest^{\ast} = \sigma^{2}\) precisely states that the test error can be made arbitrarily close to the noise floor as \(k\) (equivalently, \(r\)) goes to infinity.

% \begin{figure}[ht]
%   \centering
%   \includegraphics[width=0.48\textwidth]{figures/trade-off.pdf}
%   \caption{\label{fig:trade-off}
%     \emph{Trade-off between the testing and training errors from \Cref{proposition:trade-off}.}
%     The orange curve is the parametrized curve
%     \((x,y) =(
%     \Etrain^{\ast},
%     \Etest^{\ast}
%       )\)
%       traced out by varying \(k\) (equivalently \(r\)).
%       The resulting estimators  can achieve a continuum regimes of overfitting.
%     The blue scatter points are empirical results from synthetic experiments on the HDA model (\Cref{example:HDA}).
%     The value for \(r\) are tuned according to the tuning scheme in \Cref{remark:tuning} where the prescribed training error \(\tau\) ranges over 20 evenly spaced points in \([0,0.8]\).
% See \Cref{section:experiment} for experimental details.
%     }
% \end{figure}

\begin{figure}
  \centering
\includegraphics[width=0.42\textwidth]{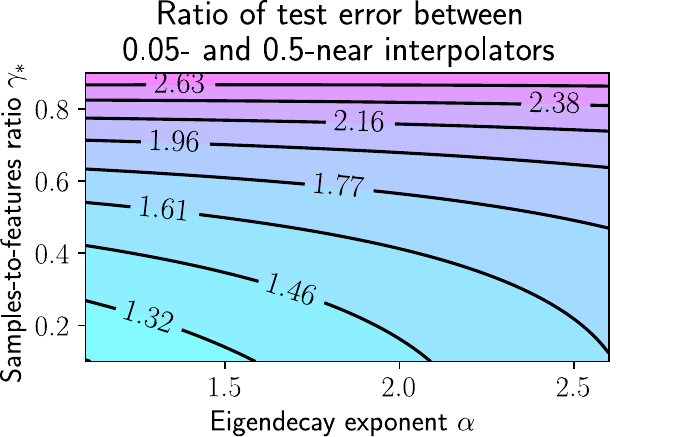}
\caption{\label{figure:steepness}\emph{Larger  power-law spectra exponent implies larger asymptotic excess test error when interpolating to \(5\%\) of the noise compared to \(50\%\) of the noise floor}. Let
\(\mathcal{E}^{\ast}_{\mathtt{test}}=\mathcal{E}^{\ast}_{\mathtt{test}}(\alpha, \gamma_\ast, \tau)
\) be as in \Cref{theorem:trade-off} where we make the dependency on parameters \(\alpha, \gamma_\ast, \tau\) explicit.
The color and contour line of  plot shows the ratio of test errors at two levels of nearness of interpolation \(
\mathcal{E}^{\ast}_{\mathtt{test}}(\alpha, \gamma_\ast, 0.05)
/\mathcal{E}^{\ast}_{\mathtt{test}}(\alpha, \gamma_\ast, 0.5)
\) over an \((\alpha,\gamma_\ast)\)-grid.
}
\end{figure}

Using \Cref{proposition:trade-off} with the implemenation of \({}_2F_1\) in \textsc{SciPy}, we illustrate the trade-off between the training error versus the test error in \Cref{fig:experiment-norms}-Right and the test error ratio in \Cref{figure:steepness}.

\begin{remark}[Data-independent regularizer-selection]\label{remark:tuning}
Let \(\tau  \in (0,\sigma^2)\) be a desired level of nearness of interpolation.
To select a regularizer \(\varrho_n\) that achieves \(\tau'\)-near-interpolation for \(\tau' \approx \tau\), we use the following method: \emph{Step 1}.\ Find \(k \in (k_{\mathtt{crit}}, +\infty)\) such that \(\Etrain^* = \Etrain^*(k) = \tau\), using the expression for \(\Etrain^*(k)\) in \Cref{proposition:trade-off}. Let \(k_\tau\) be such a \(k\).
\emph{Step 2}.\ Next, set \(r := \mathcal{R}(k_{\tau})\), where \(\mathcal{R}\) is also as in  \Cref{proposition:trade-off}. 
\emph{Step 3}.\ Set the regularizer 
 as \(\varrho_n := rn^{\alpha}\). 
\end{remark}
% \begin{remark}
%   The upshot of \Cref{proposition:trade-off} is
%   that any trade-off \((\Etrain^{\ast}, \Etest^{\ast})\) on along the  blue curve in \Cref{fig:experiment-norms}-Right  can be achieved by the tuning algorithm in \Cref{remark:tuning}.
%   For perfect-interpolators, \citet{mallinar2022benign}  shows that estimators with tempered overfitting achieve test error of exactly \(\alpha \sigma^2\).
% In contrast, \emph{near}-overfitting can achieve a continuum of test errors, i.e., \(c \sigma^2\) where \(c \in (1, c_{\max})\) belongs to a interval.
% \end{remark}

\section{RAPID NORM GROWTH}\label{section:rapid-norm-growth}
\Cref{theorem:polynomial-lower-bound} can be proven in even greater generality under random matrix-theoretic assumptions.

\begin{proposition}[Rapid norm growth under RMT]\label{theorem:polynomial-lower-bound-2}
  Suppose 
  Assumptions 
  \ref{assumption:EVD},
  \ref{assumption:self-consistent-kr}, 
  \ref{assumption:MP-law},
  and 
\ref{assumption:positivity-condition} all hold.
  % Assume that
  % the exact\footnote{
  %   This assumption can be replaced with the weaker \Cref{assumption:EVD}, which we prove in the Appendix.
  % } EVD
  % (\Cref{assumption:exact-EVD})
  % and
  % certain random matrix-theoretic\footnote{
  % The \emph{positivity condition} to be introduced in
  % \Cref{assumption:positivity-condition}.
  % }
  % conditions hold.
  For any \(r > 0\), let \(\varrho_{n} := r n^{-\alpha}\) be the regularizer for the ridge regression.
Then
\(\mathbb{E}[\|\hat{\bbeta}_{\varrho_{n}}\|_{2}^{2}] = \Omega(n^{\alpha})\).
\end{proposition}

% \section{\Cref{theorem:polynomial-lower-bound}: Rapid Norm Growth}\label{section:generic-lower-bound}
The goal of this section is to sketch the proof for \Cref{theorem:polynomial-lower-bound-2}.
Complete proofs of all results are included in the Appendix. Throughout, we assume the setting of \Cref{section:notations}.

The next step is the following:
\begin{proposition}\label{proposition:lower-bound-with-deriv-of-MP-law}
Let \(\varrho := r n^{-\alpha}\). Then we have
\[\mathbb{E}
\|
\hat{\bbeta}_{\varrho}\|_2^2
\ge
n^{\alpha}
\sigma^2
\cdot
\mathbb{E}
\big[
\frac{d }{d r}(
r \mathcal{S}_{\mathtt{esd}(n^{\alpha} \check{\bG})}(-r)
)
\big].\]
\end{proposition}
The proof of \Cref{proposition:lower-bound-with-deriv-of-MP-law}
and other omitted proofs in this section can be found in 
\Cref{appendix:section:proof-lower-bound}.

Given \Cref{proposition:lower-bound-with-deriv-of-MP-law}, the proof of \Cref{theorem:polynomial-lower-bound-2} is straightforward:

\begin{proof}[Proof of \Cref{theorem:polynomial-lower-bound-2}]

Let \[L :=
    \lim_{n\to\infty}
\mathbb{E}
\big[
\tfrac{d}{dr}( r \mathcal{S}_{\mathtt{esd}(n^{\alpha} \check{\bG})}(-r))
\big]
>0\]
be as in
\Cref{assumption:positivity-condition}.
Thus, for all \(n \gg 0\) sufficiently large, we have
\(\mathbb{E}
\big[
\tfrac{d}{dr}( r \mathcal{S}_{\mathtt{esd}(n^{\alpha} \check{\bG})}(-r))
\big]
> L/2 > 0
\).
  By
  \Cref{proposition:lower-bound-with-deriv-of-MP-law}, we get that
\(\mathbb{E}
\|
\hat{\bbeta}_{\varrho}\|_2^2
\ge
n^{\alpha}
\sigma^2
\cdot
\frac{L}{2}
\) for all \(n \gg 0\), as desired.
\end{proof}

\subsection{\(\alpha\)-scaled limiting spectral distribution}\label{section:limitingd-istribution}
 In this section, we check that power-law spectra covariance matrices has an \(\alpha\)-scaled limiting spectral distribution.
In other words, \Cref{assumption:exact-EVD} implies \Cref{assumption:EVD}. Note that this is necessary because we have proved 
\Cref{theorem:polynomial-lower-bound-2}
rather than
\Cref{theorem:polynomial-lower-bound}. So we need to make sure that the special case, i.e., \Cref{theorem:polynomial-lower-bound}, is indeed the ``special case''.
\begin{definition}
  Given a measure \(\mu\) on \(\mathbb{R}\), we let \(\mathtt{cdf}[ \mu ]\) denote the cumulative distribution function (CDF) of \(\mu\).
  % If \(\mu\) has a density, i.e.,
  % \(\mathtt{cdf}[\mu]\) is differentiable, then we let
  % \(\mathtt{pdf}[\mu]\) denote its probability density function.
\end{definition}

\begin{proposition}\label{proposition:non-negativity-of-the-ST}
Under \Cref{assumption:exact-EVD}, we have that \Cref{assumption:EVD} holds. In other words,
 % 1.
 \[\lim_{n\to\infty}
\mathtt{cdf}[ \mathtt{esd}(n^{\alpha} \bSigma) ](t)
=
\begin{cases}
1- \gamma_{\ast} t^{-1/\alpha} &: t \ge \gamma_{\ast}^{\alpha} \\
0 & \mbox{otherwise.}
\end{cases}\]

% 2. \[\lim_{n\to\infty}
% \mathtt{pdf}[ \mathtt{esd}(n^{\alpha} \bSigma) ](t)
% =
% \begin{cases}
% \frac{1}{\alpha}\gamma t^{-(1+1/\alpha)} &: t \ge \gamma^{\alpha} \\
% 0 & \mbox{otherwise.}
% \end{cases}\]

% 3. \(\lim_{n\to\infty}
% \mathcal{S}_{\mathtt{esd}(n^{\alpha} \bSigma)}(z)
% =
% \frac{\gamma}{\alpha}
% \int_{\gamma^\alpha}^{\infty}
% \frac{t^{-(1+1/\alpha)}}{t - z} dt\)

% 4. For all $z > 0$, we have that
% \[\frac{d}{dz} z \lim_{n\to \infty} \mathcal{S}_{\mathtt{esd}(n^\alpha \bSigma)}(-z) > 0.\]

\end{proposition}

\begin{figure}
  \centering
\includegraphics[width=0.47\textwidth]{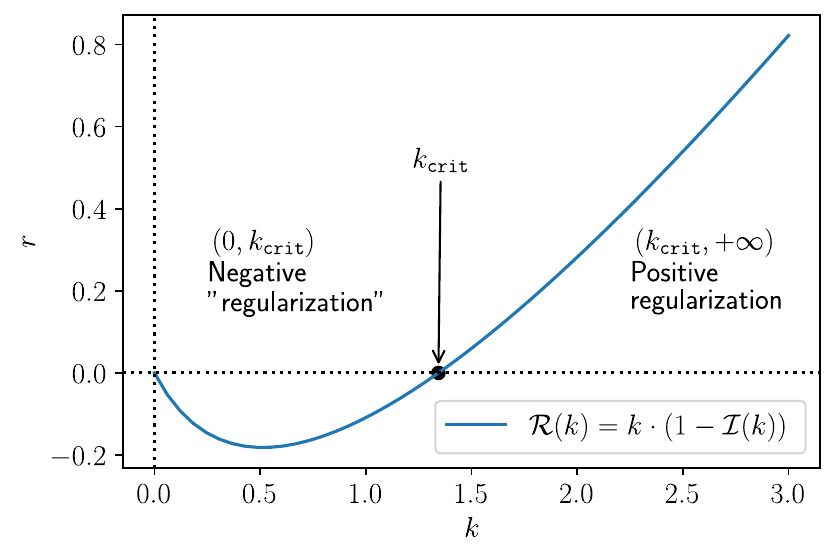}
\caption{\label{figure:R-func}\emph{The \(\mathcal{R}(k)\) function from \Cref{proposition:trade-off}}.
The \(x\)-axis is the input \(k\). Note that for \(k < k_{\mathtt{crit}}\) the regularizer \(r\) is negative.
Although we are only interested in the \((k_{\mathtt{crit}}, +\infty)\) portion, negative regularizers have been studied by \citet{tsigler2020benign,wu2020optimal}.
}
\end{figure}

\subsection{Looseness of norm-based generalization bounds}\label{section:looseness}
Conjecturally, a norm-based generalization bound for near-interpolators should have the following form: under suitable assumptions, with high probability
\[
    \sup_{\bbeta:\|\bbeta\| \le B, \Etrain(\bbeta) \le \tau} \Etest(\bbeta) =O\left(\frac{B^2 \text{Tr}(\bSigma)}{n}\right).
\]
where \(\tau \in \mathbb{R}_{\ge 0}\).
For \(\tau =0\), the best known bound for (perfect) interpolators is given by \citet[Corollary 1]{koehler2021uniform} under  Gaussian assumption on the data $X \sim \mathcal{N}(0,\bSigma)$
and $B \ge \|\bbeta^\star\|$.

To the best of our knowledge, there is no known extension to the case of near-interpolators, i.e., where \(\tau >0\).
However, such bound is not informative for our scenario, since by \Cref{theorem:trade-off} and \Cref{theorem:polynomial-lower-bound}, it is possible to choose \(\varrho_n\) such that 
\underline{\emph{1}}.\
\(\mathbb{E}[\Etrain(\hat{\bbeta}_{\varrho_n})] \to \tau\),
\underline{\emph{2}}.\
\(\Etest(\hat{\bbeta}_{\varrho_n}) \to c \in \mathbb{R}_{\ge 0}\), and 
\underline{\emph{3}}.\
$\|\bbeta_{\varrho_n}\|^2 = \Omega(n^{\alpha})$ for any \(\alpha >0\). Thus, the bound goes to infinity while \(\Etest(\hat{\bbeta}_{\varrho_n})\) is finite.

% \begin{remark}
% As a consequence of \Cref{proposition:non-negativity-of-the-ST}, we have proven that \Cref{assumption:exact-EVD} implies
%   \Cref{assumption:EVD}.
% \end{remark}

\section{EXPERIMENTS}\label{section:experiment}
We run two types of synthetic experiments. The first type, plotted in \Cref{fig:experiment-norms}, employs (linear) ridge regression.
The second type, plotted in \Cref{figure:experiment-NN}
employs neural networks. The data for both types of experiments are drawn from the HDA model (\Cref{example:HDA}).
Moreover, we have conducted experiments on several real world regression datasets from
the UCI regression collection. 

\subsection{Experiments on synthetic data}
To generate \Cref{fig:experiment-norms}-left,
we run experiments with \(\alpha \in \{1.25, 2.5\}\) and \(\tfrac{n}{p} = \gamma_\ast = \tfrac{2}{3}\).
We sweep over the train MSE parameter \(\Etrain^*\) to explore the trade-off between the train and testing MSE in linear ridge regression as described in \Cref{theorem:trade-off}.
The value of \(\Etrain^*\) are sweeped on a linearly-spaced grid of size \(16\) from \(0.05\sigma^2\) to \(0.8\sigma^2\).
The parameters are \(n_{\mathtt{train}} = 5000\), \(n_{\mathtt{test}} = 1000\), \(\gamma_{\ast} = 0.5\), \(\alpha = 1.75\) and \(\sigma^{2}=1\).

The regularizer achieving a desired training MSE \(\tau \in (0,\sigma^2)\) is chosen according to the method described in \Cref{remark:tuning}.
We sample \(\bbeta^{\star} \in \mathbb{R}^{p}\) such that \(\bbeta^{\star}_{i}\) are i.i.d Gaussian with zero mean and variance \(= 10/p\).

The same set up is used for \Cref{fig:experiment-norms}-right, except we sweep over \(n_{\mathtt{train}}\) rather than over \(\Etrain^*\).
The value of \(n_{\mathtt{train}}\) are sweeped on a logarithmically-spaced grid of size \(20\) from \(200\) to \(5000\).

For \Cref{figure:experiment-NN}, the identical set-up is used as in the \Cref{fig:experiment-norms}, except ridge regression is replaced with neural networks\footnote{We use the default settings in sklearn, except with early stopping for near-interpolation.} \emph{during training}.
We emphasize that the ground truth data (i.e., the teacher) is still generated via the same linear function \(\bbeta^\star\).
All code for the experiments are included in \Cref{section:code-appendix}.

\begin{remark}\label{remark:NN-experiments}
As discussed in the introduction, near-interpolating neural network and its interpolation-generalization trade-offs exhibit similar phenomenon as in the case of ridge regression.
Namely, larger  power-law spectra exponent implies larger asymptotic excess test error when interpolating to \(5\%\) of the noise compared to \(50\%\) of the noise floor, for instance.
\end{remark}
\begin{remark}
    Intriguely, in both the ridge regression (\Cref{fig:experiment-norms}) and neural network (\Cref{figure:experiment-NN}) experiments, the  setting with to larger value of the norm-growth exponent \(\alpha\) (right panels)
    results in poorer trade-offs (left panels). For ridge regression, this is explained by the ``Moreover'' part of \Cref{theorem:trade-off}. We believe it is an interesting future direction whether this behavior can be proved in the neural network setting.
\end{remark}

\iftoggle{conferenceVersion}{
    % FOR CONFERENCE VERSION
    
\begin{figure*}
  \centering
\includegraphics[width = 0.41\textwidth]{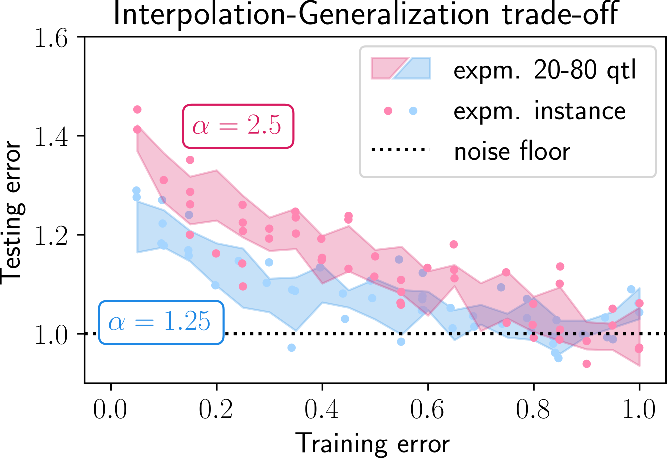}
~
~
~
\includegraphics[width = 0.43\textwidth]{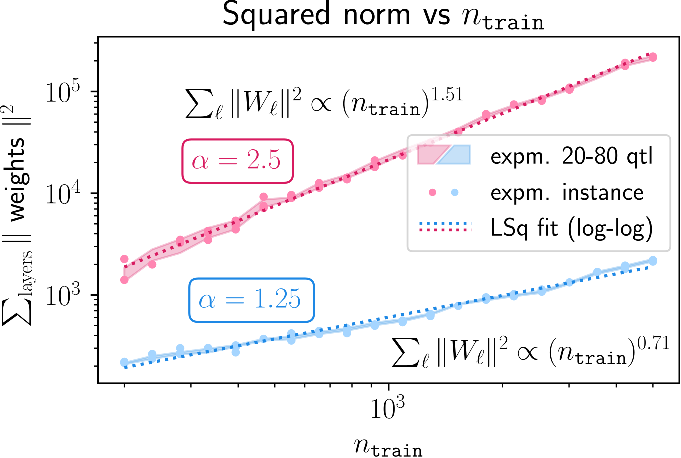}

\vspace{1em}

\includegraphics[width = 0.41\textwidth]{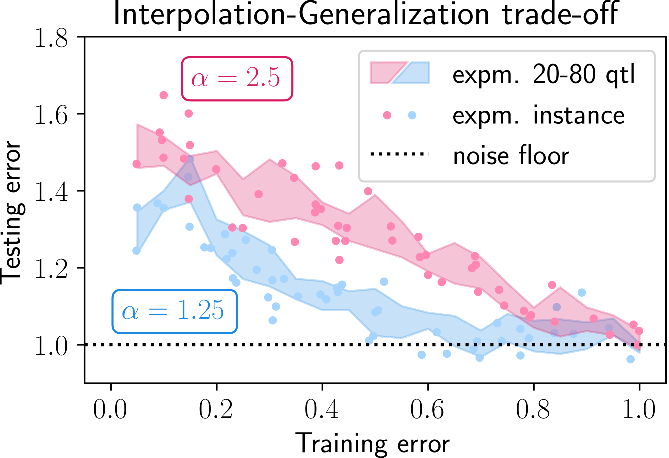}
~
~
\includegraphics[width = 0.43\textwidth]{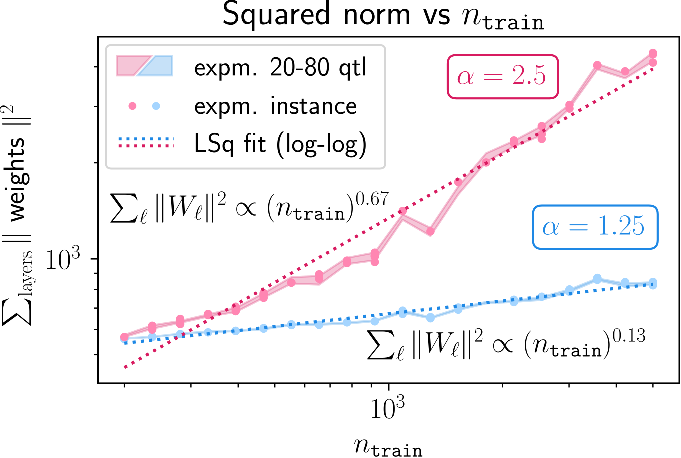}
\caption{\label{figure:experiment-NN}Experiments with neural networks (top row: 1-hidden layer, bottom row: 5-hidden layers). Analogous to \Cref{fig:experiment-norms}.
See \Cref{section:experiment} and \Cref{remark:NN-experiments} for details.}
\end{figure*}
}{
    % FOR PREPRINT VERSION
    
\begin{figure}[h]
  \centering
\includegraphics[width = 0.43\textwidth]{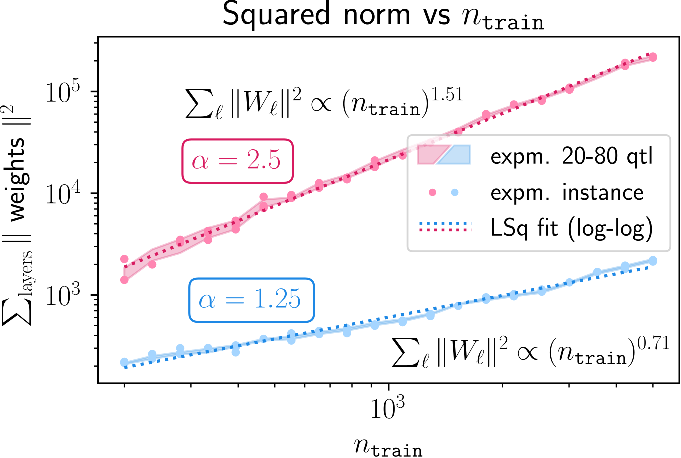}
~
~
~
\includegraphics[width = 0.41\textwidth]{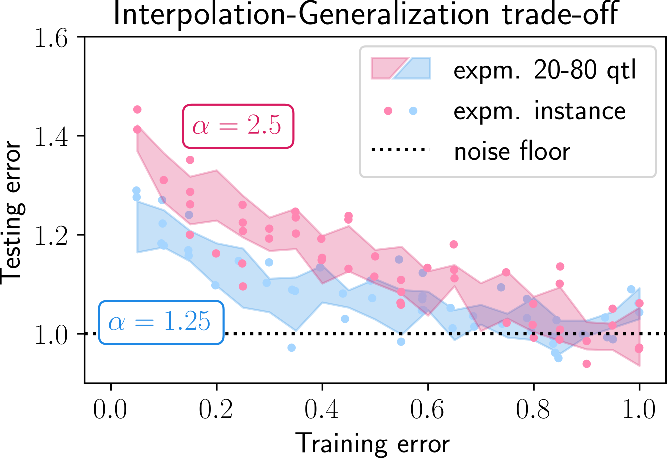}

\vspace{1em}

\includegraphics[width = 0.43\textwidth]{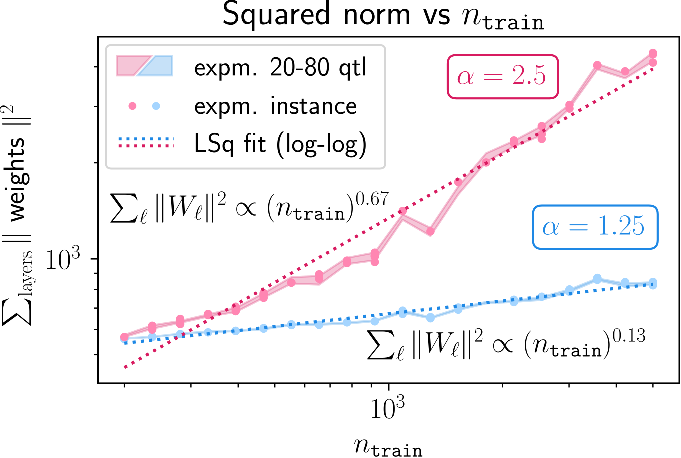}
~
~
\includegraphics[width = 0.41\textwidth]{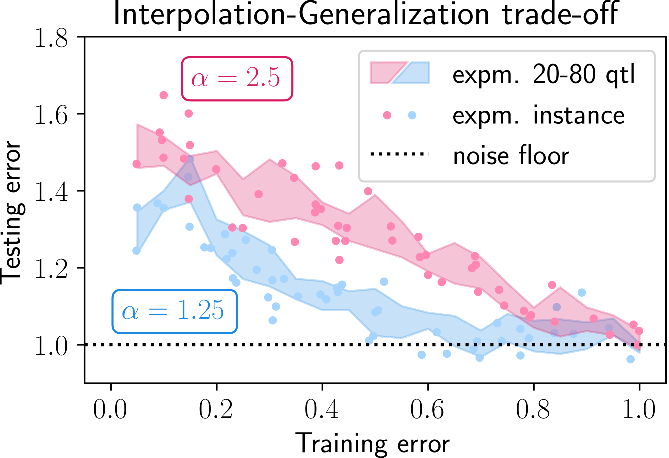}
\caption{\label{figure:experiment-NN}Experiments with neural networks (top row: 1-hidden layer, bottom row: 5-hidden layers). Analogous to \Cref{fig:experiment-norms}.
See \Cref{section:experiment} and \Cref{remark:NN-experiments} for details.}
\end{figure}
}
\subsection{Experiments on UCI datasets}

 We conduct experiments on the \texttt{forest} and the \texttt{stock} datasets from
the UCI regression collection \citep{kellyuciml}. Using neural tangent kernels \citep{arora2019harnessing}, we observe power-law spectra in the kernel matrices for both of these datasets (\Cref{figure:forest}-right and \Cref{figure:stock}-right).
In this subsection, we discuss the experiments on the \texttt{forest} dataset in relation to our theoretical results.
Due to space constraints, we refer the reader to \Cref{sec:real-world-experiment} for details of the experimental setup.

In \Cref{figure:forest}-left, note that the curve corresponding to \texttt{forest.2-1} has the fastest spectra decay and simultaneously the worst trade-off.
Evidently, larger decay exponent corresponds to a poorer trade-off, especially for near-interpolators, i.e., as the training error approaches \(0\). This is in agreement with our theoretical results under random matrix theory assumptions illustrated in \Cref{figure:steepness}.
Similar phenomenon occurs for the \texttt{stock} dataset.
See \Cref{sec:real-world-experiment}.

% \begin{minipage}[b]{1.0\linewidth}
% \includegraphics[width=0.45\linewidth]{RebuttalFigures/forest_train_test.pdf}
% ~
% \includegraphics[width=0.45\linewidth]{RebuttalFigures/forest_eigenvalues.pdf}
% \captionof{figure}{\emph{Left}. Training/testing error trade-off on the ``\texttt{forest}'' dataset from the UCI regression dataset collection using kernel ridge regression with the neural tangent kernel.
% Each curve is labeled by ``\texttt{DatasetName.d-f}''
% where ``\texttt{d}'' and ``\texttt{f}'' represents the number
% of layers and the number of \emph{fixed} layers 
% in the NTK corresponding to ReLU networks.
% \emph{Right}. The eigenvalue index vs eigenvalue plot of the NTK matrix exhibits power-law spectra. A tiny value is added to the eigenvalues to prevent skewed graph.
% }
% \end{minipage}

\section{ADDITIONAL RELATED WORKS AND NOVELTY OF OUR WORK}\label{section:more-related-works}

\emph{Technical Novelty of Theorem \ref{theorem:polynomial-lower-bound}.}  Prior works of \cite{Hastie2019SurprisesIH, Derezinski2019ExactEF} require a lower bound on the smallest eigenvalue of the covariance matrix. Hence, the scenario studied in this paper is not amenable to their results. Further, Theorem \ref{theorem:trade-off} shows that when we have a sharper decay (larger $\alpha$) of the eigenvalues, we have a worse tradeoff between the test and training error. Hence, we show that the scenario from \cite{Hastie2019SurprisesIH} is the most benign one. This well conditioning assumption is relaxed in \cite{cheng2022dimension}, however, they require that $\|\Sigma^{-1/2}\beta\|$ is finite. We do not require this assumption. Finally, \cite{dobriban2018high} does not need the well-conditioning assumption but instead needs to assume  an isotropic distribution on $\beta$. Since we work with fixed $\beta$, their results do not apply.

\emph{Trade-offs in interpolation-based learning}. Prior works \citep{ghosh2022universal, belkin2018understand,sonthalia2023under} have studied the tradeoff between near interpolation and generalization. 
For regression,  previous works have also studied the fundamental trade-off in learning algorithms between
overparametrization and (Lipschitz) smoothness
\citep{bubeck2021universal}, and
robustness and smoothness
\citep{zhang2022many}.

\emph{Loosenss of existing generalization bounds}.
\citet[Theorem 1]{belkin2018understand} establishes for \emph{classification} that the RKHS norm of a ``near-interpolating'' classifier
grows at rate  \(\Omega(\exp(n^{1/p}))\).
The growth is  unbounded if \(n = \Omega(\exp(p))\). If  the number of samples \(n = \Theta(\mathrm{poly}(p))\), then the lower bound does not grow to infinity. While our results are for regression and thus not directly comparable, our lower bound  is meaningful in the more practical \(n \propto p\) regime. 

\noindent\emph{Power-law spectra datasets}.
% \label{sec:empirical-evidence}
Synthetic data with artificial power law EVD covariance have been used frequently as toy examples
\citep{berthier2020tight,mallinar2022benign}.
On real datasets, power law EVD is often observed to describe neural tangent kernels (NTK) well in practice, including on
\textsc{MNIST}
(\citep[Fig,\ 4]{bahri2021explaining}
and
\citep[Fig.\ 2]{velikanov2022tight}),
\textsc{Fashion-MNIST}
\citep[Fig.\ 7]{cui2021generalization}
\textsc{Caltech 101}
\citep[Fig.\ 1]{murray2022characterizing},
\textsc{CIFAR-100} \citep[Fig.\ 3]{wei2022more}.

\begin{figure*}
  \centering
\includegraphics[width = 0.44\textwidth]{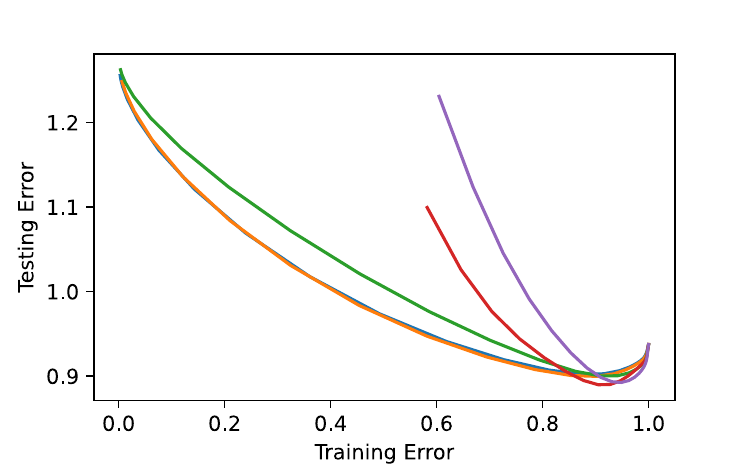}
~
~
~
\includegraphics[width = 0.42\textwidth]{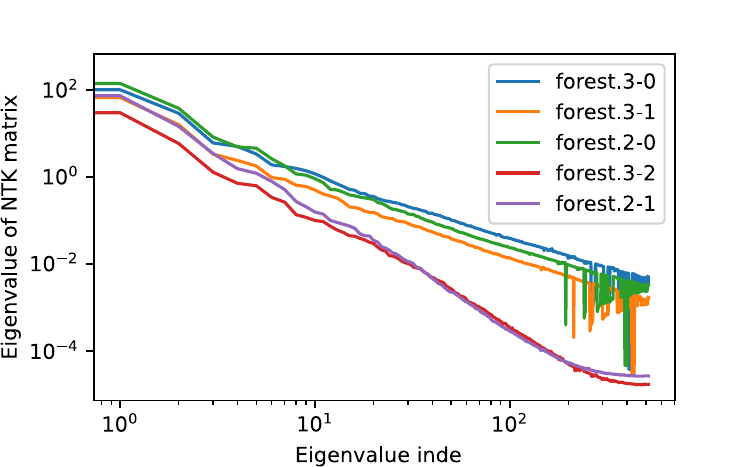}
\caption{\emph{Left}. Training/testing error trade-off on the ``\texttt{forest}'' dataset from the UCI regression dataset collection using kernel ridge regression with the neural tangent kernel.
Each curve is labeled by ``\texttt{DatasetName.d-f}''
where ``\texttt{d}'' and ``\texttt{f}'' represents the number
of layers and the number of \emph{fixed} layers 
in the NTK corresponding to ReLU networks.
\emph{Right}. The eigenvalue index vs eigenvalue plot of the NTK matrix exhibits power-law spectra. A tiny value is added to the eigenvalues  for better visualization on the log-scale.
}\label{figure:forest}
\end{figure*}

\emph{Theoretical machine learning works using power-law spectra}.
\citet{bordelon2020spectrum} shows that power law EVD implies power law learning curve.
\citet[\S 6.2]{velikanov2021explicit} computes the power law EVD exponent for certain NTKs with ReLU to be \(\alpha = 1 + \frac{1}{d}\).
\citet{murray2022characterizing} computes the EVD for NTKs with several different activations.
\citet[Theorem 6]{bartlett2020benign}
shows that benign overfitting occurs when the covariance matrix eigenvalues \(\lambda_i = i^{-1} \log^{-b}(i+1)\) for \(b > 1\).
\citet{mallinar2022benign} studies power law decay for \(\alpha \ge 1\) and proposes a taxonomy of overfitting into three categories: catastropic, tempered and benign. The EVD condition is also known as
the \emph{capacity condition} in the kernel ridge regression literature.
See
\citet{bietti2021sample}
and the references there-in.

\emph{Random matrix theory (RMT)}. The signal processing research community have long been using RMT for theoretical analysis \citep{couillet2012signal}.
Increasingly RMT has been applied to machine learning as well as a key tool for analysis.

In particular, \citet{dobriban2018high,Hastie2019SurprisesIH,jacot2020kernel,liang2020just} have applied RMT for (kernel) ridge regression, \citet{sonthalia2023training, kausik2023generalization} use it to understand generalization of linear denoisers, 
\citet{paquette2022halting,paquette2021sgd} uses the so-called local Marchenko-Pastur law  \citep{knowles2017anisotropic} to analyze gradient-based algorithms.
Finally, \citet{wei2022more} also applies such local law to analyze the so-called \emph{generalized cross-
validation (GCV) estimator}.

\section{DISCUSSION AND LIMITATIONS}

We conclude with several future research directions that we believe will be fruitful:

\noindent \emph{Connection to early stopping}.
Typically, early stopping prevents the trained algorithm from perfectly interpolating the data. Can early stopped learning theory results,
e.g., \citet{ji2021early,kuzborskij2022learning}, be applied to analyze near-interpolators?

\noindent \emph{Near-interpolators and uniform convergence generalization bound}.
Is possible to use uniform convergence-based approach to give non-vacuous generalization bound under the setting studied in this work?
This question has already been raised by  \citet{dobriban2018high} in the context of classification in a similar setting.
An interesting question is if classical learning theory can be used to obtain results that are currently only obtained via random matrix theoretic or similar techniques.
Another approach is to extend the results of \citet{koehler2021uniform} to the near-interpolation setting.

% \noindent \emph{Closing the gap with the lower bound}.
% Is the trade-off in \Cref{proposition:trade-off} optimal? In \Cref{fig:gap}, we see that there is a large gap between the curve traced out by
% \((x,y) =(
%     \Etrain^{\ast},
%     \Etest^{\ast}
%       )\)
%       and
% the tightest (and only) known universal lower bound \citep[Eqn. (29)]{ghosh2022universal}.
% We note that the
% Marchenko-Pastur lower bound
%   \citep[Eqn. (29)]{ghosh2022universal}
%   does not consider eigenvalue decay and thus is  not a true lower bound for our EVD setting.
% An interesting future work would be to determine if there exists a lower bound for the EVD setting that is tight.

\textbf{Limitations}. Our work is restricted to analyzing a random matrix model. Understanding the phenomenon uncovered in this paper in more general models and additional real world settings will be needed. Moreover, our work does not rule out the existence of uniform convergence generalization bound.

\iftoggle{conferenceVersion}{
\textbf{Code availability}. Code used to run and plot the experiments shown in all figures is available at \url{https://github.com/YutongWangUMich/Near-Interpolators-Figures}.

\subsubsection*{Acknowledgements}
YW acknowledges support from the Eric and Wendy Schmidt AI in Science Postdoctoral
Fellowship, a Schmidt Futures program.
WH acknowledges support from the Google Research Scholar program.
}{}

\bibliography{references}

\begin{thebibliography}{}

\bibitem[Arora et~al., 2019]{arora2019harnessing}
Arora, S., Du, S.~S., Li, Z., Salakhutdinov, R., Wang, R., and Yu, D. (2019).
\newblock Harnessing the power of infinitely wide deep nets on small-data
  tasks.
\newblock In {\em International Conference on Learning Representations}.

\bibitem[Arous and Guionnet, 2008]{arous2008spectrum}
Arous, G.~B. and Guionnet, A. (2008).
\newblock The spectrum of heavy tailed random matrices.
\newblock {\em Communications in Mathematical Physics}, 278(3):715--751.

\bibitem[Bahri et~al., 2021]{bahri2021explaining}
Bahri, Y., Dyer, E., Kaplan, J., Lee, J., and Sharma, U. (2021).
\newblock Explaining neural scaling laws.
\newblock {\em arXiv preprint arXiv:2102.06701}.

\bibitem[Bai and Silverstein, 2010]{bai2010spectral}
Bai, Z. and Silverstein, J.~W. (2010).
\newblock {\em Spectral analysis of large dimensional random matrices},
  volume~20.
\newblock Springer.

\bibitem[Bartlett et~al., 2020]{bartlett2020benign}
Bartlett, P.~L., Long, P.~M., Lugosi, G., and Tsigler, A. (2020).
\newblock Benign overfitting in linear regression.
\newblock {\em Proceedings of the National Academy of Sciences},
  117(48):30063--30070.

\bibitem[Belkin et~al., 2018]{belkin2018understand}
Belkin, M., Ma, S., and Mandal, S. (2018).
\newblock To understand deep learning we need to understand kernel learning.
\newblock In {\em International Conference on Machine Learning}, pages
  541--549. PMLR.

\bibitem[Berthier et~al., 2020]{berthier2020tight}
Berthier, R., Bach, F., and Gaillard, P. (2020).
\newblock Tight nonparametric convergence rates for stochastic gradient descent
  under the noiseless linear model.
\newblock {\em Advances in Neural Information Processing Systems},
  33:2576--2586.

\bibitem[Bietti et~al., 2021]{bietti2021sample}
Bietti, A., Venturi, L., and Bruna, J. (2021).
\newblock On the sample complexity of learning with geometric stability.
\newblock In {\em Advances in Neural Information Processing Systems}.

\bibitem[Bordelon et~al., 2020]{bordelon2020spectrum}
Bordelon, B., Canatar, A., and Pehlevan, C. (2020).
\newblock Spectrum dependent learning curves in kernel regression and wide
  neural networks.
\newblock In {\em International Conference on Machine Learning}, pages
  1024--1034. PMLR.

\bibitem[Bubeck and Sellke, 2021]{bubeck2021universal}
Bubeck, S. and Sellke, M. (2021).
\newblock A universal law of robustness via isoperimetry.
\newblock {\em Advances in Neural Information Processing Systems},
  34:28811--28822.

\bibitem[Canatar et~al., 2021]{canatar2021spectral}
Canatar, A., Bordelon, B., and Pehlevan, C. (2021).
\newblock Spectral bias and task-model alignment explain generalization in
  kernel regression and infinitely wide neural networks.
\newblock {\em Nature communications}, 12(1):1--12.

\bibitem[Cheng and Montanari, 2022]{cheng2022dimension}
Cheng, C. and Montanari, A. (2022).
\newblock Dimension free ridge regression.
\newblock {\em arXiv preprint arXiv:2210.08571}.

\bibitem[Couillet and Debbah, 2012]{couillet2012signal}
Couillet, R. and Debbah, M. (2012).
\newblock Signal processing in large systems: A new paradigm.
\newblock {\em IEEE Signal Processing Magazine}, 30(1):24--39.

\bibitem[Cui et~al., 2021]{cui2021generalization}
Cui, H., Loureiro, B., Krzakala, F., and Zdeborov{\'{a}}, L. (2021).
\newblock Generalization error rates in kernel regression: The crossover from
  the noiseless to noisy regime.
\newblock In {\em Advances in Neural Information Processing Systems}, pages
  10131--10143.

\bibitem[Derezinski et~al., 2019]{Derezinski2019ExactEF}
Derezinski, M., Liang, F.~T., and Mahoney, M.~W. (2019).
\newblock Exact expressions for double descent and implicit regularization via
  surrogate random design.
\newblock {\em ArXiv}, abs/1912.04533.

\bibitem[Dobriban and Wager, 2018]{dobriban2018high}
Dobriban, E. and Wager, S. (2018).
\newblock High-dimensional asymptotics of prediction: Ridge regression and
  classification.
\newblock {\em The Annals of Statistics}, 46(1):247--279.

\bibitem[Dutka, 1984]{dutka1984early}
Dutka, J. (1984).
\newblock The early history of the hypergeometric function.
\newblock {\em Archive for History of Exact Sciences}, pages 15--34.

\bibitem[Ghosh and Belkin, 2022]{ghosh2022universal}
Ghosh, N. and Belkin, M. (2022).
\newblock A universal trade-off between the model size, test loss, and training
  loss of linear predictors.
\newblock {\em arXiv preprint arXiv:2207.11621}.

\bibitem[Goel and Klivans, 2017]{goel2017eigenvalue}
Goel, S. and Klivans, A. (2017).
\newblock Eigenvalue decay implies polynomial-time learnability for neural
  networks.
\newblock {\em Advances in Neural Information Processing Systems}, 30.

\bibitem[Hastie et~al., 2022]{Hastie2019SurprisesIH}
Hastie, T., Montanari, A., Rosset, S., and Tibshirani, R.~J. (2022).
\newblock {Surprises in {H}igh-{D}imensional {R}idgeless {L}east {S}quares
  {I}nterpolation}.
\newblock {\em The Annals of Statistics}.

\bibitem[Jacot et~al., 2020a]{jacot2020implicit}
Jacot, A., Simsek, B., Spadaro, F., Hongler, C., and Gabriel, F. (2020a).
\newblock Implicit regularization of random feature models.
\newblock In {\em International Conference on Machine Learning}, pages
  4631--4640. PMLR.

\bibitem[Jacot et~al., 2020b]{jacot2020kernel}
Jacot, A., Simsek, B., Spadaro, F., Hongler, C., and Gabriel, F. (2020b).
\newblock Kernel alignment risk estimator: Risk prediction from training data.
\newblock {\em Advances in Neural Information Processing Systems},
  33:15568--15578.

\bibitem[Ji et~al., 2021]{ji2021early}
Ji, Z., Li, J., and Telgarsky, M. (2021).
\newblock Early-stopped neural networks are consistent.
\newblock {\em Advances in Neural Information Processing Systems},
  34:1805--1817.

\bibitem[Karp and L{\'o}pez, 2017]{karp2017representations}
Karp, D.~B. and L{\'o}pez, J.~L. (2017).
\newblock Representations of hypergeometric functions for arbitrary parameter
  values and their use.
\newblock {\em Journal of Approximation Theory}, 218:42--70.

\bibitem[Kausik et~al., 2023]{kausik2023generalization}
Kausik, C., Srivastava, K., and Sonthalia, R. (2023).
\newblock Generalization error without independence: Denoising, linear
  regression, and transfer learning.
\newblock {\em arXiv preprint arXiv:2305.17297}.

\bibitem[Kelly et~al., 2023]{kellyuciml}
Kelly, M., Longjohn, R., and Nottingham, K. (2023).
\newblock The uci machine learning repository.
\newblock \url{https://archive.ics.uci.edu}.

\bibitem[Knowles and Yin, 2017]{knowles2017anisotropic}
Knowles, A. and Yin, J. (2017).
\newblock Anisotropic local laws for random matrices.
\newblock {\em Probability Theory and Related Fields}, 169(1):257--352.

\bibitem[Koehler et~al., 2021]{koehler2021uniform}
Koehler, F., Zhou, L., Sutherland, D.~J., and Srebro, N. (2021).
\newblock Uniform convergence of interpolators: Gaussian width, norm bounds and
  benign overfitting.
\newblock {\em Advances in Neural Information Processing Systems},
  34:20657--20668.

\bibitem[Kuzborskij and Szepesv{\'a}ri, 2022]{kuzborskij2022learning}
Kuzborskij, I. and Szepesv{\'a}ri, C. (2022).
\newblock Learning lipschitz functions by gd-trained shallow overparameterized
  relu neural networks.
\newblock {\em arXiv preprint arXiv:2212.13848}.

\bibitem[Liang and Rakhlin, 2020]{liang2020just}
Liang, T. and Rakhlin, A. (2020).
\newblock Just interpolate: Kernel “ridgeless” regression can generalize.
\newblock {\em The Annals of Statistics}, 48(3):1329--1347.

\bibitem[Mahoney and Martin, 2019]{mahoney2019traditional}
Mahoney, M. and Martin, C. (2019).
\newblock Traditional and heavy tailed self regularization in neural network
  models.
\newblock In {\em International Conference on Machine Learning}, pages
  4284--4293. PMLR.

\bibitem[Mallinar et~al., 2022]{mallinar2022benign}
Mallinar, N.~R., Simon, J.~B., Abedsoltan, A., Pandit, P., Belkin, M., and
  Nakkiran, P. (2022).
\newblock Benign, tempered, or catastrophic: Toward a refined taxonomy of
  overfitting.
\newblock In {\em Advances in Neural Information Processing Systems}.

\bibitem[Murray et~al., 2022]{murray2022characterizing}
Murray, M., Jin, H., Bowman, B., and Montufar, G. (2022).
\newblock Characterizing the spectrum of the {NTK} via a power series
  expansion.
\newblock {\em arXiv preprint arXiv:2211.07844}.

\bibitem[Nakkiran et~al., 2020]{nakkiran2020optimal}
Nakkiran, P., Venkat, P., Kakade, S.~M., and Ma, T. (2020).
\newblock Optimal regularization can mitigate double descent.
\newblock In {\em International Conference on Learning Representations}.

\bibitem[Paquette et~al., 2021]{paquette2021sgd}
Paquette, C., Lee, K., Pedregosa, F., and Paquette, E. (2021).
\newblock Sgd in the large: Average-case analysis, asymptotics, and stepsize
  criticality.
\newblock In {\em Conference on Learning Theory}, pages 3548--3626. PMLR.

\bibitem[Paquette et~al., 2022]{paquette2022halting}
Paquette, C., van Merri{\"e}nboer, B., Paquette, E., and Pedregosa, F. (2022).
\newblock Halting time is predictable for large models: A universality property
  and average-case analysis.
\newblock {\em Foundations of Computational Mathematics}, pages 1--77.

\bibitem[Silverstein and Choi, 1995]{silverstein1995analysis}
Silverstein, J.~W. and Choi, S.-I. (1995).
\newblock Analysis of the limiting spectral distribution of large dimensional
  random matrices.
\newblock {\em Journal of Multivariate Analysis}, 54(2):295--309.

\bibitem[Simon et~al., 2023]{simon2021neural}
Simon, J.~B., Dickens, M., Karkada, D., and Deweese, M. (2023).
\newblock The eigenlearning framework: A conservation law perspective on kernel
  ridge regression and wide neural networks.
\newblock {\em Transactions on Machine Learning Research}.

\bibitem[Sonthalia et~al., 2023]{sonthalia2023under}
Sonthalia, R., Li, X., and Gu, B. (2023).
\newblock Under-parameterized double descent for ridge regularized least
  squares denoising of data on a line.
\newblock {\em arXiv preprint arXiv:2305.14689}.

\bibitem[Sonthalia and Nadakuditi, 2023]{sonthalia2023training}
Sonthalia, R. and Nadakuditi, R.~R. (2023).
\newblock Training data size induced double descent for denoising feedforward
  neural networks and the role of training noise.
\newblock {\em Transactions on Machine Learning Research}.

\bibitem[Tao, 2011]{tao2011stieltjes}
Tao, T. (2011).
\newblock Intuitive understanding of the {S}tieltjes transform.
\newblock MathOverflow.
\newblock Version: 2011-10-25.

\bibitem[Tsigler and Bartlett, 2020]{tsigler2020benign}
Tsigler, A. and Bartlett, P.~L. (2020).
\newblock Benign overfitting in ridge regression.
\newblock {\em arXiv preprint arXiv:2009.14286}.

\bibitem[Velikanov and Yarotsky, 2021]{velikanov2021explicit}
Velikanov, M. and Yarotsky, D. (2021).
\newblock Explicit loss asymptotics in the gradient descent training of neural
  networks.
\newblock {\em Advances in Neural Information Processing Systems},
  34:2570--2582.

\bibitem[Velikanov and Yarotsky, 2022]{velikanov2022tight}
Velikanov, M. and Yarotsky, D. (2022).
\newblock Tight convergence rate bounds for optimization under power law
  spectral conditions.
\newblock {\em arXiv preprint arXiv:2202.00992}.

\bibitem[Wang et~al., 2024]{wang2024spectral}
Wang, Z., Engel, A., Sarwate, A.~D., Dumitriu, I., and Chiang, T. (2024).
\newblock Spectral evolution and invariance in linear-width neural networks.
\newblock {\em Advances in Neural Information Processing Systems}, 36.

\bibitem[Wei et~al., 2022]{wei2022more}
Wei, A., Hu, W., and Steinhardt, J. (2022).
\newblock More than a toy: Random matrix models predict how real-world neural
  representations generalize.
\newblock In {\em Proceedings of the 39th International Conference on Machine
  Learning}, pages 23549--23588. PMLR.

\bibitem[Wu and Xu, 2020]{wu2020optimal}
Wu, D. and Xu, J. (2020).
\newblock On the optimal weighted {$\ell_2$} regularization in
  overparameterized linear regression.
\newblock {\em Advances in Neural Information Processing Systems},
  33:10112--10123.

\bibitem[Zhang et~al., 2017]{zhang2017understanding}
Zhang, C., Bengio, S., Hardt, M., Recht, B., and Vinyals, O. (2017).
\newblock Understanding deep learning requires rethinking generalization.
\newblock In {\em International Conference on Learning Representations}.

\bibitem[Zhang et~al., 2021]{zhang2021understanding}
Zhang, C., Bengio, S., Hardt, M., Recht, B., and Vinyals, O. (2021).
\newblock Understanding deep learning (still) requires rethinking
  generalization.
\newblock {\em Communications of the ACM}, 64(3):107--115.

\bibitem[Zhang et~al., 2022]{zhang2022many}
Zhang, H., Wu, Y., and Huang, H. (2022).
\newblock How many data are needed for robust learning?
\newblock {\em arXiv preprint arXiv:2202.11592}.

\end{thebibliography}
% \bibliographystyle{plainnat}

%%%%%%%%%%%%%%%%%%%%%%%%%%%%%%%%%%%%%%%%%%%%%%%%%%%%%%%%%%%%

\iftoggle{conferenceVersion}{
    % FOR CONFERENCE VERSION
% %%% BEGIN INSTRUCTIONS %%%
\section*{Checklist}

 \begin{enumerate}

 \item For all models and algorithms presented, check if you include:
 \begin{enumerate}
   \item A clear description of the mathematical setting, assumptions, algorithm, and/or model. [Yes]
   \item An analysis of the properties and complexity (time, space, sample size) of any algorithm. [Not Applicable]
   \item (Optional) Anonymized source code, with specification of all dependencies, including external libraries. [Yes]
 \end{enumerate}

 \item For any theoretical claim, check if you include:
 \begin{enumerate}
   \item Statements of the full set of assumptions of all theoretical results. [Yes]
   \item Complete proofs of all theoretical results. [Yes]
   \item Clear explanations of any assumptions. [Yes]     
 \end{enumerate}

 \item For all figures and tables that present empirical results, check if you include:
 \begin{enumerate}
   \item The code, data, and instructions needed to reproduce the main experimental results (either in the supplemental material or as a URL). [Yes]
   \item All the training details (e.g., data splits, hyperparameters, how they were chosen). [Yes]
         \item A clear definition of the specific measure or statistics and error bars (e.g., with respect to the random seed after running experiments multiple times). [Yes]
         \item A description of the computing infrastructure used. (e.g., type of GPUs, internal cluster, or cloud provider). [Yes]
 \end{enumerate}

 \item If you are using existing assets (e.g., code, data, models) or curating/releasing new assets, check if you include:
 \begin{enumerate}
   \item Citations of the creator If your work uses existing assets. [Not Applicable]
   \item The license information of the assets, if applicable. [Not Applicable]
   \item New assets either in the supplemental material or as a URL, if applicable. [Not Applicable]
   \item Information about consent from data providers/curators. [Not Applicable]
   \item Discussion of sensible content if applicable, e.g., personally identifiable information or offensive content. [Not Applicable]
 \end{enumerate}

 \item If you used crowdsourcing or conducted research with human subjects, check if you include:
 \begin{enumerate}
   \item The full text of instructions given to participants and screenshots. [Not Applicable]
   \item Descriptions of potential participant risks, with links to Institutional Review Board (IRB) approvals if applicable. [Not Applicable]
   \item The estimated hourly wage paid to participants and the total amount spent on participant compensation. [Not Applicable]
 \end{enumerate}

 \end{enumerate}

}{
    % FOR PREPRINT VERSION
}

\newpage
\appendix
\onecolumn

\section*{Appendix}

\section{Proof of \Cref{theorem:trade-off} --- the exact trade-off formula}
Our goal is to calculate the
asymptotic test error 
$\Etest^{\ast}$ under the assumptions of \Cref{theorem:trade-off}.
This is accomplished through the following three steps.

The first step is to calculate the closed-form solution for the integrals defined in
\Cref{definition:IJ-functions} which are key ingredients for the expression of $\Etest^{\ast}$.
This is done in \Cref{section:closed-form-appendix}.
The second step is to relate the integrals from \Cref{definition:IJ-functions}
to the self-consistent equations in 
\Cref{equation:self-consistent-kr}.
This is done in \Cref{section:exact-evd}.
The final step is to relate the self-consistent equations \Cref{equation:self-consistent-kr} to 
the
asymptotic test error 
$\Etest^{\ast}$.
This is done in \Cref{sec:appendix:eigenlearning}.

\subsection{Closed-form expression for the integrals in \Cref{definition:IJ-functions}}\label{section:closed-form-appendix}
We prove the identities
\[
\int_{0}^{\tfrac{1}{\gamma_{\ast}}}
    % (1+k{x}^{\alpha})^{-1}dx,
    \frac{dx}{1+ k {x}^{\alpha} }=
\gamma_{\ast}^{-1} F(1, \tfrac{1}{\alpha}; 1 + \tfrac{1}{\alpha}; -k \gamma_{\ast}^{-\alpha}),\,\, \mbox{and}
\]
\[
    \int_{0}^{\tfrac{1}{\gamma_{\ast}}}
    \frac{dx}{(1+ k {x}^{\alpha} )^{2}}=
\gamma_{\ast}^{-1} F(2, \tfrac{1}{\alpha}; 1 + \tfrac{1}{\alpha}; -k \gamma_{\ast}^{-\alpha}).
\]
as shown in the main text following \Cref{definition:IJ-functions}.
\begin{proof}

Let \({}_{2}F_{1}(a,b;c;z)\) be the \emph{Gauss hypergeometric function}.
Note that the function can be implemented in \textsc{SciPy} as \texttt{scipy.special.hyp2f1} and is used to plot
\Cref{fig:experiment-norms}.
Let \(\Gamma\) denote the Gamma function.
The integral representation of the Gaussian hypergeometric function is well-known and is given by\footnote{
  We used
  the formula stated in
\cite{karp2017representations}.
}
\begin{equation}
  {}_{2}F_{1}(\sigma,a;b;-z) =
  \frac{\Gamma(b)}{\Gamma(a) \Gamma(b-a)} \int_{0}^{1}
  \frac{t^{a-1} (1-t)^{b-a-1}}{(1+zt)^{\sigma}} dt.
\label{equation:hypergeometric-function-integral-rep}
\end{equation}
Moreover, \eqref{equation:hypergeometric-function-integral-rep} is finite for \(z \in \mathbb{C} \setminus (-\infty, -1]\) and \(\Re(b-a) >0 \) and \(\Re(a) > 0\).
Thus, by
\Cref{equation:hypergeometric-function-integral-rep}, we have
\[
  {}_{2}F_{1}(1, \tfrac{1}{\alpha}; 1 + \tfrac{1}{\alpha}; -k \gamma_{\ast}^{-\alpha})
  =
\frac{\Gamma(1+\tfrac{1}{\alpha})}{\Gamma(\tfrac{1}{\alpha}) \Gamma(1)}
\int_{0}^{1}
\frac{t^{1/\alpha} t^{-1}}{1+k \gamma_{\ast}^{-\alpha}t} dt
=
  \frac{1}{\alpha}
\int_{0}^{1}
  \frac{t^{1/\alpha} t^{-1}}{1+k \gamma_{\ast}^{-\alpha}t} dt
\]
where the second inequality follows from the well-known identity \(z  = \Gamma(1+z)/\Gamma(z)\) for the Gamma function.
Let \(u = t^{1/\alpha}\). Then we have \(du =
\frac{1}{\alpha}t^{1/\alpha}t^{-1} dt\). Thus, by u-substitution, we have
    \[
  \frac{1}{\alpha}
\int_{0}^{1}
  \frac{t^{1/\alpha} t^{-1}}{1+k \gamma_{\ast}^{-\alpha}t} dt
  =
\int_{0}^{1}
  \frac{1}{1+k \gamma_{\ast}^{-\alpha}u^{\alpha}} du
  =
\int_{0}^{1/\gamma_{\ast}}
  \frac{\gamma_{\ast}}{1+k x^{\alpha}} dx
    \]
    where the second inequality used u-substituted with \(x = u/\gamma_{\ast}\).
    Now, we have by the definition of \(\mathcal{I}\) in \Cref{definition:IJ-functions} that
    \[
\mathcal{I}(k)
=
\int_{0}^{1/\gamma_{\ast}}
  \frac{1}{1+k x^{\alpha}} dx
  =
  \frac{1}{\gamma_{\ast}}
\int_{0}^{1/\gamma_{\ast}}
  \frac{\gamma_{\ast}}{1+k x^{\alpha}} dx
  =
  \frac{1}{\gamma_{\ast}}
  {}_{2}F_{1}(1, \tfrac{1}{\alpha}; 1 + \tfrac{1}{\alpha}; -k \gamma_{\ast}^{-\alpha})
    \]
    as desired.
    By an analogous calculation, we get
    \(\mathcal{J}(k)
=
  \frac{1}{\gamma_{\ast}}
  {}_{2}F_{1}(2, \tfrac{1}{\alpha}; 1 + \tfrac{1}{\alpha}; -k \gamma_{\ast}^{-\alpha})
\).
\end{proof}

\subsection{Proof of
\Cref{proposition:trade-off}}\label{section:exact-evd}

We begin by analyzing the functions defined in \Cref{definition:IJ-functions}
and prove the items 1 and 2 of the ``Moreover'' part of \Cref{proposition:trade-off}:

\begin{proposition}\label{proposition:I-and-J}
  Let \(\mathcal{I}\) and \(\mathcal{J}\) be functions as defined in \Cref{definition:IJ-functions}.
  Under
  \Cref{assumption:exact-EVD} and
  \Cref{assumption:self-consistent-kr}, we have that
\(  r = \mathcal{R}(k)
  :=
  k \cdot (1  - \mathcal{I}(k))
  \)
  and
\(  \tfrac{dr}{dk}
  =
  1 - \mathcal{J}(k)
\).

Furthermore, the following holds:
\begin{enumerate}

  \item \(\mathcal{R}(k) \asymp k\) for \(k \gg 0\),

  \item There exists \(k_{\mathtt{crit}} >0\) such that \(\mathcal{R}(k_{\mathtt{crit}}) = 0\),
        \(\mathcal{R}\) is increasing and positive on
        \((k_{\mathtt{crit}}, +\infty)\).
  \item \(\mathcal{J}(k) < 1\) for \(k \in (k_{\mathtt{crit}},+\infty)\) and \(\mathcal{J}(+\infty)  = 0\).
\end{enumerate}
  \end{proposition}

% \begin{proof}[Proof of \Cref{proposition:I-and-J}]
%   First, we rewrite the limit in \Cref{equation:self-consistent-kr} as
%   \[
%     \lim_{n\to\infty} \frac{1}{n} \sum_{i=1}^{p}
%     \frac{1}{1+ k n^{-\alpha} \sigma_{i}^{-1}}
%     =
%     \lim_{n\to\infty} \frac{1}{n} \sum_{i=1}^{n/\gamma}
%     \frac{1}{1+ k {(i/n)}^{\alpha} }
%     =
%     \int_{0}^{1/\gamma_{\ast}}
%     \frac{dx}{1+ k {x}^{\alpha} }
%   \]
% The right-most equality follows from the definition of the (Riemann) integral.
% If \(\gamma_{\ast} = 0\), then \(1/\gamma_{\ast} = +\infty\) and the above is interpreted as an improper Riemann integral.
% Now, rearranging \Cref{equation:self-consistent-kr}, we get the desired formula.

% For the ``furthermore'' part, we have
% \[
%   \frac{dr}{dk}
%   =
%   1 - \mathcal{I}(k;\alpha,\gamma_{\ast})
%   -
%   k
%   \int_{0}^{1/\gamma_{\ast}}
%   \frac{d}{dk}
%   \left( \frac{1}{1 + kx^{\alpha}} \right)
%   dx
% \]
% Since
% \[
%   \frac{d}{dk}
%   \left( \frac{1}{1 + kx^{\alpha}} \right)
%   =
%   -\frac{x^{\alpha}}{(1+kx^{\alpha})^{2}}
% \]
% we have
% \[
%   \frac{dr}{dk}
%   =
%   1 - \mathcal{I}(k;\alpha,\gamma_{\ast})
%   +
%   \int_{0}^{1/\gamma_{\ast}}
%    \frac{kx^{\alpha}}{(1 + kx^{\alpha})^{2}}
%   dx
% \]
% Since
% \[
% \mathcal{I}(k;\alpha,\gamma_{\ast})
% =
%     \int_{0}^{1/\gamma_{\ast}}
%     \frac{dx}{1+ k {x}^{\alpha} }
%     =
%     \int_{0}^{1/\gamma_{\ast}}
%     \frac{1+k{x}^{\alpha}}{(1+ k {x}^{\alpha} )^{2}}dx
% \]
% we have
% \[
%   \frac{dr}{dk}
%   =
%   1 -
%   \int_{0}^{1/\gamma_{\ast}}
%    \frac{1}{(1 + kx^{\alpha})^{2}}
%   dx
% \]
% as desired.
% \end{proof}

\begin{proof}[Proof of \Cref{proposition:I-and-J}]
We begin by proving the first part: that
\(  r = \mathcal{R}(k)
  :=
  k \cdot (1  - \mathcal{I}(k))
  \)
  and
\(  \tfrac{dr}{dk}
  =
  1 - \mathcal{J}(k)
\). Rewrite the limit in \Cref{equation:self-consistent-kr} as follows:
  \[
    \lim_{n\to\infty} \frac{1}{n} \sum_{i=1}^{p}
    \frac{1}{1+ k n^{-\alpha} \lambda_{i}^{-1}}
    =
    \lim_{n\to\infty} \frac{1}{n} \sum_{i=1}^{n/\gamma}
    \frac{1}{1+ k {(i/n)}^{\alpha} }
    =
    \int_{0}^{1/\gamma_{\ast}}
    \frac{dx}{1+ k {x}^{\alpha} }
  \]
The right-most equality follows from the definition of the (Riemann) integral.
If \(\gamma_{\ast} = 0\), then \(1/\gamma_{\ast} = +\infty\) and the above is interpreted as an improper Riemann integral.
Now, rearranging \Cref{equation:self-consistent-kr}, we get the desired formula of \(  r = \mathcal{R}(k)
  :=
  k \cdot (1  - \mathcal{I}(k))
  \).
%   For every \(k >0\) and \(\alpha >1\), the function \(x \mapsto \frac{1}{1+kx^{\alpha}}\) is continuous and integrable on \(x \in \mathbb{R}_{\ge 0}\).
%   Now, the desired identity follows immediately from
%   \(   \lim_{n\to \infty}
%   \frac{1}{n} \sum_{i=1}^{p}
%   \frac{\lambda_{i}}{kn^{-\alpha} + \lambda_{i}}
%   =
%   \lim_{n\to \infty}
%   \frac{1}{n} \sum_{i=1}^{p}
%   \frac{1}{1+k(i/n)^{\alpha}}
%  \)
% and the definition of Riemann integration.
The formula for \(\frac{dr}{dk}\) follows by ``differentiating under the integral'' (Leibniz integral rule).
% Note that this also proves the assertion made in
% \Cref{remark:rk-relationship}.

For the first item of the ``Furthermore'' part, it suffices to show that \(\lim_{k\to+\infty} \mathcal{I}(k) =0\). This follows from the fact that \( \lim_{k \to +\infty}\frac{1}{1 + kx^{\alpha}} =0 \) for all \(x > 0\), integrability of the function \((1+x^{\alpha})^{-1}\) over \(\mathbb{R}_{\ge 0}\), and the dominated convergence theorem.
Likewise, \(\lim_{k\to \infty}\mathcal{J}(k) =0\) as well.

For the second item of the ``Furthermore'' part, we note that for all \(x\) sufficiently large, we have
\(\tfrac{dr}{dk} > 0\) since \(\lim_{k\to \infty}\mathcal{J}(k) =0\).
Now, let \(k_{\mathtt{crit}}\) be the largest real number such that \(\mathcal{R}(k_{\mathtt{crit}}) = 0\).
Since \(\mathcal{R}(0) = 0\), we must have \(k_{\mathtt{crit}} \ge 0\).

For all \(k > k_{\mathtt{crit}}\), we claim that \(\mathcal{I}(k) < 1\). To see this, assume the contrary. Then by the fact that \(\lim_{k\to+\infty} \mathcal{I}(k) = 0\) and the intermediate value theorem, there must exists \(k'\) such that
\(k' > k\) such that
\(\mathcal{I}(k') = 1\) which implies that \(\mathcal{R}(k') = 0\). This contradicts the maximality of \(k_{\mathtt{crit}}\).

Finally, since \(1+kx^{\alpha} \le (1+ kx^{\alpha})^{2}\) for all \(k \ge 0\) and \(x \ge 0\), we have that \(\mathcal{I}(k) \ge \mathcal{J}(k)\) for all such \(k\)'s.
Thus, by the previous claim, for all \(k > k_{\mathtt{crit}}\), we have
\(1 > \mathcal{I}(k) \ge \mathcal{J}(k)\). This proves that
\(\frac{dr}{dk} >0\) for all \(k > k_{\mathtt{crit}}\), as desired.
\end{proof}

\subsection{Review of the eigenlearning framework \citep{simon2021neural}
}\label{sec:appendix:eigenlearning}

Before proceeding with finishing the proof of \Cref{proposition:trade-off}, we briefly review the eigenlearning framework.
\citet{simon2021neural} calculates the test error for the estimator
\begin{equation}
  \check{\bbeta}_{\delta} :=
  \bX(\bX^{\top}\bX + \delta \bI_{n})^{-1}y
  =
  \bX(n\check{\bG} + \delta \bI_{n})^{-1}y
\label{equation:simon}
\end{equation}
for kernel ridge regression using the so-called \emph{eigenlearning equations} \citep[Section 4.1]{simon2021neural}.
Below, we recall  some relevant parts of the framework:
\begin{definition}[Eigenlearning eqn.\ specialized to setting in \Cref{section:setup}]\label{definition:eigenlearning-framework}
  Suppose that the ground truth regression function is linear, i.e., \(f(x) = x^{\top} \bbeta^{\star}\) for some \(\bbeta^{\star} \in \mathbb{R}^{p}\).
Let \(\delta\) and \(\kappa\) satisfy
the equation
\begin{equation}
  n = \frac{\delta}{\kappa} +   \sum_{i=1}^{p} \frac{\lambda_{i}}{\lambda_{i} + \kappa}.
\label{equation:simon-self-consistent-eqn}
\end{equation}
Define the following \(n\)-dependent quantities:
  \begin{enumerate}
    \item
\emph{Overfitting coefficient}:
\(  \mathcal{E}_{\mathtt{coef}} := n \frac{d \kappa}{d \delta}
\)

    \item
\emph{Testing error}:
          \(\Etest := \mathcal{E}_{\mathtt{coef}} (\sigma^{2} + C)\)
          where \[C = \textstyle \sum_{i=1}^{p} (1- \mathcal{L}_{i}) (\beta^{\star}_{i})^{2}
          \quad\mbox{and}\quad \mathcal{L}_{i} := \frac{\lambda_{i}}{\lambda_{i} + \kappa}.\]

    \item \emph{Training error}:
          \(  \Etrain
  := \frac{\delta^{2}} {n^{2} \kappa^{2}} \Etest
\).
  \end{enumerate}
\end{definition}
% \subsection{Proof of \Cref{proposition:trade-off}}
\begin{proof}
  [Proof of \Cref{proposition:trade-off}]
% Throughout this section, we assume that we are in the situation of \Cref{proposition:trade-off}.
% Now,
\citet{simon2021neural} uses a different scaling for ridge regression than the one we use. To bridge the different notations, we first resolve this discrepancy.
Comparing \Cref{equation:simon} with the expression
in
\Cref{lemma:RR-solution},
 if we let \(\delta := n \varrho\), then the expressions are equivalent, i.e.,
\(
\check{\beta}_{\delta} = \hat{\beta}_{\varrho}
\). To see this, note that
\begin{align*}
\check{\beta}_{\delta}
=
\check{\beta}_{n\varrho}
&=
  X(X^{\top}X + n \varrho\mathbb{I}_{n})^{-1}y
\\&=
  (XX^{\top} + n \varrho\mathbb{I}_{p})^{-1}Xy
  \quad \because
  \mbox{\Cref{lemma:woodbury}}
                           \\& =
  (n (n^{-1}XX^{\top} +  \varrho\mathbb{I}_{p}))^{-1}Xy
                           \\& =
  ( \hat{\Sigma} +  \varrho\mathbb{I}_{p})^{-1}\tfrac{1}{n}Xy
  = \hat{\beta}_{\varrho} \quad \because{\mbox{Definition of \(\hat{\beta}_{\varrho}\)}}
\end{align*}
Below, let \(r>0\) be arbitrary.
Furthermore, we claim that as \({n\to\infty}\), we have \(r,k\) satisfies
\Cref{equation:self-consistent-kr} if and only if
the tuple \((\delta , \kappa ) := ( nr n^{-\alpha},  kn^{-\alpha})\) satisfies \Cref{equation:simon-self-consistent-eqn}:
\begin{align*}
&n = \frac{\delta}{\kappa} + \sum_{i=1}^{p} \frac{\lambda_{i}}{\lambda_{i} + \kappa}
    \iff
  n = \frac{nrn^{-\alpha}}{kn^{-\alpha}} + \sum_{i=1}^{p} \frac{\lambda_{i}}{\lambda_{i} + kn^{-\alpha}}
    \iff
  1 = \frac{r}{k} +\frac{1}{n} \sum_{i=1}^{p} \frac{1}{1+ kn^{-\alpha}\lambda_{i}^{-1}}.
  % \lael{equation:}
\end{align*}
Taking limit as \(n\to\infty\), we have proved the claim.

Next, we show that \(\lim_{n\to\infty}C =0\) where \(C\) is as in \Cref{definition:eigenlearning-framework}.
We have
\(\mathcal{L}_{i} := \frac{\lambda_{i}}{\lambda_{i} + \kappa}
=
\frac{1}{1 + k (i/n)^{\alpha}}
\).
Note that \(\lim_{n\to\infty} \mathcal{L}_{i} = 1\) for all fixed \(i\).
On the other hand, since
\(\sup_{n=1,2\dots}\|\beta^{\star}\|_{2} < +\infty\), dominated convergence theorem implies that \(\lim_{n\to\infty} C = 0\)

We claim that the following asymptotic expression for the testing and training error hold:
\begin{equation}
    \label{equation:asymp-tr-te}
  \lim_{n\to\infty} \mathcal{E}_{\mathtt{test}} =
  \sigma^{2}
\cdot
  \tfrac{dk}{dr}
  \quad \mbox{and}
  \quad
  \lim_{n\to\infty}
  \mathcal{E}_{\mathtt{train}}
  =
  \sigma^{2} \cdot \tfrac{r^{2}}{k^{2}}\cdot\tfrac{dk}{dr}
\end{equation}
where \(r\) and \(k\)
satisfy
\Cref{equation:self-consistent-kr}
from from \Cref{assumption:self-consistent-kr}.

To see this, first note that the overfitting coefficient satisfies
\[
  \mathcal{E}_{\mathtt{coef}} := n \tfrac{d \kappa}{d \delta}
  =
  n \tfrac{d \kappa}{d \varrho} \tfrac{d \varrho}{d \delta}
  =
  n \tfrac{d \kappa}{d \varrho} \tfrac{1}{n}
  =
  \tfrac{d \kappa}{d \varrho}
  =
  \tfrac{d k}{dr}.
\]

Thus, we obtain the following asymptotic expression
\[
  \lim_{n\to\infty} \mathcal{E}_{\mathtt{test}} = \mathcal{E}_{\mathtt{coef}} \cdot \sigma^{2} =
  \sigma^{2}
\cdot
  \tfrac{dk}{dr}.
\]

On the other hand, the training error is given by
\[
  \mathcal{E}_{\mathtt{train}}
  = \tfrac{\delta^{2}} {n^{2} \kappa^{2}} \mathcal{E}_{\mathtt{test}}
  = \tfrac{\varrho^{2}} { \kappa^{2}} \mathcal{E}_{\mathtt{test}}
  =
  \mathcal{E}_{\mathtt{test}}
  \cdot
  \tfrac{r^{2}} {k^{2}}.
\]
Therefore,
\(  \lim_{n\to\infty}
  \mathcal{E}_{\mathtt{train}}
  =
  \sigma^{2} \cdot \tfrac{r^{2}}{k^{2}}\cdot\tfrac{dk}{dr}
  \).
  This proves
    \eqref{equation:asymp-tr-te}, as desired.
\end{proof}

% \section{Proof of results from \Cref{section:exact-evd}}

% (Part 2). Simply take the derivative of the expression from part 1.

% (Part 3). This is simply the definition of the Stieltjes transform.
% That the limit and the Stieltjes transform commute follows from
% \cite[Theorem B.9]{bai2010spectral}.

% (Part 4). From part 3, we have
% $$
%  z \lim_{n\to \infty} \mathcal{S}_{\mathtt{esd}(n^\alpha \bSigma)}(-z)
% =
% \frac{\gamma}{\alpha}
% \int_{\gamma^\alpha}^{\infty}
% \frac{zt^{-(1+1/\alpha)}}{t + z} dt
% $$

% Taking derivative of the both side, we have
% $$
% \frac{d}{dz}
%  z \lim_{n\to \infty} \mathcal{S}_{\mathtt{esd}(n^\alpha \bSigma)}(-z)
% =
% \frac{\gamma}{\alpha}
% \int_{\gamma^\alpha}^{\infty}
% t^{-(1+1/\alpha)}
% \frac{d}{dz}
% \frac{z}{t + z} dt
% =
% \frac{\gamma}{\alpha}
% \int_{\gamma^\alpha}^{\infty}
% t^{-(1+1/\alpha)}
% \frac{t}{(t + z)^2} dt
% =
% \frac{\gamma}{\alpha}
% \int_{\gamma^\alpha}^{\infty}
% \frac{t^{-1/\alpha}}{(t + z)^2} dt.
% $$

% For all $z > 0$, the integrand is positive. Hence, we have that
% \[\frac{d}{dz}
%  z \lim_{n\to \infty} \mathcal{S}_{\mathtt{esd}(n^\alpha \bSigma)}(-z)
%  >
%  0.\]
%

\section{Proof of \Cref{theorem:polynomial-lower-bound-2} --- rapid norm growth under RMT assumptions
}\label{appendix:section:proof-lower-bound}

The first key technical step the following:
\begin{proposition}\label{proposition:generic-lower-bound-helper}
  \( \mathbb{E}
\| \hat{\bbeta}_{\varrho}\|_2^2
\ge
n^{-1}
\sigma^2
\mathbb{E}[
\mathrm{tr}(
(\hat{\bSigma}  + \varrho \bI_{p})^{-2}
\hat{\bSigma})]
 \).
\end{proposition}
\begin{proof}
  Below, for brevity we let
\(\ba := \bX^{\top} \bbeta^{\star}\) and
   \(\bM :=
  ( \hat{\bSigma}  + \varrho \bI_{p})^{-1}
\frac{1}{n}
\bX
  \).
  Recall the closed-form solution for \Cref{equation:KRR} is given by the formula
\begin{equation}
\hat{\bbeta}_{\varrho} :=
( \hat{\bSigma}  + \varrho \bI_{p})^{-1}
\tfrac{1}{n}
\bX y.
\label{lemma:RR-solution}
\end{equation}
Thus,
  \[  \hat{\bbeta}_{\varrho} =
( \hat{\bSigma}  + \varrho \bI_{p})^{-1}
\frac{1}{n}
\bX y
=
  ( \hat{\bSigma}  + \varrho \bI_{p})^{-1}
\frac{1}{n}
\bX
(f(\bX) + {\bvarepsilon})
=
\bM(\ba+\bvarepsilon).\]
  Thus,
  \[
\| \hat{\bbeta}_{\varrho}\|_2^2
=
(\ba+\bvarepsilon)^{\top}\bM^{\top}\bM (\ba+\bvarepsilon)
=
\underbrace{\ba^{\top} \bM^{\top} \bM \ba}_{\ge 0}
+
\bvarepsilon^{\top} \bM^{\top} \bM \bvarepsilon
+
2\bvarepsilon^{\top} \bM^{\top} \bM \ba.
  \]
  Note that \(\bvarepsilon \perp \bM^{\top} \bM \ba\) since \(\bvarepsilon \perp \bX\).
Thus, since \(\mathbb{E}[\bvarepsilon] = 0\), we have
  \[
\mathbb{E}[\| \hat{\bbeta}_{\varrho}\|_2^2]
=
\mathbb{E}[(\ba+\bvarepsilon)^{\top}\bM^{\top}\bM (\ba+\bvarepsilon)
]
\ge
\mathbb{E}[
\bvarepsilon^{\top} \bM^{\top} \bM \bvarepsilon
]
=
\mathbb{E}[
\mathrm{tr}
(
\bM^{\top} \bM \bvarepsilon\bvarepsilon^{\top}
)
].
  \]
  Since \(\bM^{\top} \bM \perp \bvarepsilon \bvarepsilon^{\top}\), we have
  \[
\mathbb{E}[
\mathrm{tr}
(
\bM^{\top} \bM \bvarepsilon\bvarepsilon^{\top}
)
]
=
\mathrm{tr}
(
\mathbb{E}[
\bM^{\top} \bM]
\mathbb{E}[\bvarepsilon\bvarepsilon^{\top}
]
)
=
\mathrm{tr}
(
\mathbb{E}[
\bM^{\top} \bM
\sigma^{2} \bI_{n}
]
)
=
\sigma^{2}
\mathbb{E}[
\mathrm{tr}
(
\bM^{\top} \bM)].
  \]
  On the other hand,
  \(     \bM^{\top} \bM
    =
    \frac{1}{n}
( \hat{\bSigma}  + \varrho \bI_{p})^{-1}
\hat{\bSigma}
( \hat{\bSigma}  + \varrho \bI_{p})^{-1}
\).
Using the cyclic property of trace, we get the desired inequality.
\end{proof}

\begin{proof}[Proof sketch of \cref{proposition:generic-lower-bound-helper}]
  We first simplify \(\|\hat{\bbeta}_{\varrho}\|_{2}^{2}\) using the well-known formula for ridge regression:
Next, let
   \(\bM :=
  ( \hat{\bSigma}  + \varrho \bI_{p})^{-1}
\tfrac{1}{n}
\bX
  \).
Using the independence of \(\bX\)  and \(\bvarepsilon\), we get
  \(\mathbb{E}[\| \hat{\bbeta}_{\varrho}\|_2^2]
\ge
\mathbb{E}[
\mathrm{tr}
(
\bM^{\top} \bM \bvarepsilon\bvarepsilon^{\top}
)
]
\).
  Since \(\bM^{\top} \bM \) and \(\bvarepsilon \bvarepsilon^{\top}\) are also independent, we have
  \[
\mathbb{E}[
\mathrm{tr}
(
\bM^{\top} \bM \bvarepsilon\bvarepsilon^{\top}
)
]
=
\sigma^{2}
\mathbb{E}[
\mathrm{tr}
(
\bM^{\top} \bM)].
  \]
  By
  \(     \bM^{\top} \bM
    =
    \frac{1}{n}
( \hat{\bSigma}  + \varrho \bI_{p})^{-1}
\hat{\bSigma}
( \hat{\bSigma}  + \varrho \bI_{p})^{-1}
\)
and the cyclic property of trace, we get the desired inequality.
\end{proof}

For the sake of completeness, we prove
\Cref{lemma:RR-solution} though its well-known
\begin{proof}[Proof of \Cref{lemma:RR-solution}]
  Start with the objective function
\( \mathcal{F}(\bbeta):= \frac{1}{n} \| \bX^\top \bbeta - \by\|_2^2 + \varrho \| \bbeta \|_2^2 \).
Take derivative with respect to $\bbeta$, we have
\begin{align*}
  &
\frac{1}{2}
\nabla_{\bbeta}
\left(
 \frac{1}{n} \|\bX^\top \bbeta - \by\|_2^2 + \varrho \|\bbeta\|_2^2
 \right)
 =
\frac{1}{2}
\nabla_{\bbeta}
\left(
\bbeta^\top (\hat{\bSigma}   + \varrho \bI_{p})\bbeta
-
\frac{2}{n}
\bbeta^\top \bX \by
\right)
  \\
  &
=
(  \hat{\bSigma} + \varrho \bI_{p})\bbeta
-
\frac{1}{n}
\bX \by.
\end{align*}
Since \(\nabla_{\bbeta} \mathcal{F}(\hat{\bbeta}_{\varrho}) = 0\), we are done.
\end{proof}

\begin{lemma}[Special case of Woodbury formula]\label{lemma:woodbury}
  Let \begin{align}
\bM \in \mathbb{R}^{p\times n}
\end{align} be an arbitrary matrix and
  \(\varrho \in (0,\infty)\).
  Then
  \[
    (\bM\bM^{\top} + \varrho \bI_{p})^{-1} \bM
    =
    \bM (\bM^{\top} \bM + \varrho \bI_{n})^{-1} \in \mathbb{R}^{n \times p}.
  \]
\end{lemma}

\begin{proof}[Proof of \Cref{lemma:woodbury}]
  It suffices to prove \Cref{lemma:woodbury} for the special case when \(\varrho = 1\), which we assume below.
  By the Woodbury matrix identity, we have
  \begin{equation}
    (\bM\bM^{\top} + \bI_{p} )^{-1}
    =
    \bI - \bM (\bM^{\top}\bM + \bI_{n})^{-1}\bM^{\top}
\label{equation:lemma:woodbury-1}
\end{equation}
For brevity, let \(\bP := \bM\bM^{\top} + \bI_{p}\)
and let
\(\bN := \bM^{\top}\bM + \bI_{n}\).
To proceed, we have
  \begin{align*}
    &\bP^{-1}\bM\\
    &=
    \bM - \bM \bN^{-1}\bM^{\top}\bM
      \quad \because \mbox{Multiplying \eqref{equation:lemma:woodbury-1} by \(\bM\) on the right}
      \\& =
    \bM(
    \bI_{n} -  \bN^{-1}\bM^{\top}\bM)
    \quad \because \mbox{Factoring out \(\bM\) on the left}
    \\&=
    \bM(\bI_{n} - (\bI_{n} - \bN^{-1}))
    \quad \because \mbox{\(\bI_{n} = \bN^{-1}\bN =
    \bN^{-1}+\bN^{-1}\bM^{\top}\bM\)}
    \\&=
    \bM \bN^{-1}
\end{align*}
as desired.
\end{proof}

\begin{lemma}\label{lemma:derivative-formula-for-stieltjes}
  Let $\bM\in \mathbb{R}^{p\times p}$ be any symmetric matrix and \(z \in \mathbb{R}\). Then we have \[ \tfrac{d}{dz} \mathrm{tr}(z( \bM + z \bI_p)^{-1})
=
\mathrm{tr}(\bM( \bM + z \bI_p)^{-2}).\]
\end{lemma}
\begin{proof}[Proof of \Cref{lemma:derivative-formula-for-stieltjes}]
Without the loss of generality, suppose that $\bM = \mathrm{diag}(\lambda_{i},\dots, \lambda_{p})$.
  Then we have
$
    f(z):=
\mathrm{tr}(z(\bM + z \bI_{p})^{-1})
=
\sum_{i=1}^{p} \frac{z}{\lambda_{i} +z}.
 $
   Now, from elementary calculus, we have
  \[ \frac{d}{dx} \frac{x}{y +x}
  =
  (y+x)^{-1} - x (y+x)^{-2}
  =
  (y+x)^{-2}((y+x)-x)
  =
  \frac{y}{(y+x)^{2}}.\]
From this, we recover the fact that
$
    \frac{d}{d z}
    f(z)
=
\sum_{i=1}^{n} \frac{\lambda_{i}}{(\lambda_{i} + z)^{2}}
=
\mathrm{tr}(\bM(\bM +z \bI_p)^{-2}),
$
  as desired.
\end{proof}

\begin{lemma}[Gram-to-covariance]\label{lemma:gram-to-covariance}
Let $c \in \mathbb{R}$ and \(z \in \mathbb{C}\) be arbitrary, then
\( \mathcal{S}_{\mathtt{esd}(c\hat{\bSigma})}(z)
=
\gamma \cdot
\mathcal{S}_{\mathtt{esd}(c\check{\bG})}(z)
-
\frac{(1-\gamma) }{z}\).
\end{lemma}

\begin{proof}[Proof of \Cref{lemma:gram-to-covariance}]
  Without the loss of generality, we may assume that $c = 1$.
Let $\hat{\lambda}_1 \ge \dots \ge \hat{\lambda}_p$ be the eigenvalues of $\hat{\bSigma}$. Since $p > n$, we necessarily have that
$
\hat{\lambda}_{n+1} = \cdots = \hat{\lambda}_p = 0$.
Moreover, $\hat{\lambda}_1,\dots, \hat{\lambda}_n$ are the eigenvalues of $\check{\bG}$.
Now, unwinding the definition, we have

\[ \mathcal{S}_{\mathtt{esd}(\hat{\bSigma})}(z)
=
\frac{1}{p}
\sum_{i=1}^{p}
\frac{1}{\hat{\lambda}_i  - z} \]
and
\[\mathcal{S}_{\mathtt{esd}(\check{\bG})}(z)
=
\frac{1}{n}
\sum_{i=1}^{n}
\frac{1}{\hat{\lambda}_i  - z}.\]

Thus,
\begin{align*}
\mathcal{S}_{\mathtt{esd}(\hat{\bSigma})}(z)
&=
\frac{1}{p}
\left(
\sum_{i=1}^{n}
\frac{1}{\hat{\lambda}_i  - z}
+
\sum_{i=n+1}^{p}
\frac{1}{- z}
\right)
\\&=
\left(
\frac{n}{p}
\frac{1}{n}
\sum_{i=1}^{n}
\frac{1}{\hat{\lambda}_i  - z}
\right)
-
\frac{p-n}{p}
\frac{1}{z}
\\&=
\gamma
\cdot
\mathcal{S}_{\mathtt{esd}(\check{\bG})}(z)
-
\frac{(1-\gamma)}{z}
\end{align*}
as desired.
\end{proof}

\begin{proof}[Proof of \Cref{proposition:lower-bound-with-deriv-of-MP-law}]
  Recall from
  \Cref{proposition:generic-lower-bound-helper}
that
\(\mathbb{E}
\|
\hat{\bbeta}\|_2^2
\ge
n^{-1}
\sigma^2
\mathbb{E}[
\mathrm{tr}(
(\hat{\bSigma}  + \varrho \bI_{p})^{-2}
\hat{\bSigma})]
\).
Below, we analyze the term inside the expectation.
By the definition of the Stieltjes transform, we have
\[ \mathrm{tr}( \varrho(\hat{\bSigma} + \varrho \bI_p)^{-1})
=
\mathrm{tr}( r n^{-\alpha} (\hat{\bSigma} + r n^{-\alpha}\bI_p)^{-1})
=
\mathrm{tr}( r  (n^{\alpha}\hat{\bSigma} + r \bI_p)^{-1})
=
pr \mathcal{S}_{\mathtt{esd}(n^{\alpha} \hat{\bSigma})}(-r).\]
Therefore, by
\Cref{lemma:derivative-formula-for-stieltjes}, we have
\begin{align*}
& \frac{d }{d r}\left(
p
r \mathcal{S}_{\mathtt{esd}(n^{\alpha} \hat{\bSigma})}(-r)
\right)
=
\frac{d }{d r}
\mathrm{tr}( \varrho(\hat{\bSigma} + \varrho \bI_p)^{-1})
\\
&
=
\frac{d \varrho}{d r} \cdot
\frac{d }{d \varrho}
\mathrm{tr}( \varrho(\hat{\bSigma} + \varrho\bI_p)^{-1})
=
n^{-\alpha}
\mathrm{tr}((\hat{\bSigma} + \varrho \bI_p)^{-2} \hat{\bSigma}).
\end{align*}

By \Cref{lemma:gram-to-covariance}, we have
\begin{align*}
  &
p
r \mathcal{S}_{\mathtt{esd}(n^{\alpha} \hat{\bSigma})}(-r)
=
pr
\left(
\gamma
\cdot
\mathcal{S}_{\mathtt{esd}(n^{\alpha}\check{\bG})}(-r)
+
\frac{(1-\gamma)}{r}
\right)
    \\&
=
n
r
\mathcal{S}_{\mathtt{esd}(n^{\alpha}\check{\bG})}(-r)
+
p(1-\gamma)
\end{align*}
Thus, we have
\[
  \frac{d}{dr}
  \left(
  p
  r \mathcal{S}_{\mathtt{esd}(n^{\alpha} \hat{\bSigma})}(-r)
\right  )
  =
n
  \frac{d}{dr}
  \left(
  r
  \mathcal{S}_{\mathtt{esd}(n^{\alpha}\check{\bG})}(-r)
\right)
\]
from which we conclude that
\[
\mathrm{tr}((\hat{\bSigma} + \varrho \bI_p)^{-2} \hat{\bSigma})
=
n^{\alpha+1}
  \frac{d}{dr}
  \left(
  r
  \mathcal{S}_{\mathtt{esd}(n^{\alpha}\check{\bG})}(-r)
\right).
\]
In view of
\(\mathbb{E}
\|
\hat{\bbeta}\|_2^2
\ge
n^{-1}
\sigma^2
\mathbb{E}[
\mathrm{tr}(
(\hat{\bSigma}  + \varrho \bI_{p})^{-2}
\hat{\bSigma})]
\) from \Cref{proposition:generic-lower-bound-helper},
we get the desired inequality.
\end{proof}

\subsection{Proof of \Cref{proposition:non-negativity-of-the-ST}}

\begin{proof}[Proof of \Cref{proposition:non-negativity-of-the-ST}]
  To simplify notations in this proof, we write \(\gamma\) instead of \(\gamma_{\ast}\).
  Now, the set of eigenvalues of \(n^{\alpha}\bSigma\) can be expressed as
\begin{align*}
  &
\{
(n/i)^\alpha
\}_{i=1,\dots, p}
    \\
  &
    =
% \mathrm{pow}(
\{
\underbrace{(\tfrac{n}{p})^\alpha}_{=\gamma^\alpha},
,
\dots,
(\tfrac{n}{n+1})^\alpha,
\underbrace{\tfrac{n}{n}}_{=1}, ( \tfrac{n}{n-1})^\alpha, \dots, \underbrace{(\tfrac{n}{1})^\alpha}_{=n^\alpha}\}.
% ,
% \alpha)
\end{align*}
Thus,
\(\mathtt{cdf}[\mathtt{esd}(n^{\alpha} \bSigma)](t) = 0\) if \(t < \gamma^{\alpha}\)
and
\(= 1\) if \(t \ge n^{\alpha}\).
It remains to calculate
\(\mathtt{cdf}[\mathtt{esd}(n^{\alpha} \bSigma)](t) = 0\) if \(t < \gamma^{\alpha}\)
for \(t \in [\gamma^{\alpha}, n^{\alpha}]\).

To this end, let \(t \in [\gamma^{\alpha}, n^{\alpha})\)
and
\(j(t) \in \{1,\dots, p\}\) be the smallest index such that
\(
(n/j(t))^{\alpha} \le t
\).
By definition of the CDF, we have
\(
\mathtt{cdf}[\mathtt{esd}(n^{\alpha} \bSigma)](t)
=
\tfrac{1}{p}(
p- j(t) +1)
\).
We first argue that \(j(t) \ne 1\) by contradiction.
If \(j(t)=1\), then we have \(n^{\alpha} \le t\). Since \(t \in [\gamma^{\alpha}, n^{\alpha}]\), this implies that \(t = n^{\alpha}\), a contradiction.
Thus, \(j(t) \ne 1\).

Now, by the definition of \(j(t)\), we have
\(
(n/(j(t)-1))^{\alpha} > t
\).
Therefore,
\(
n/j(t)
\le
t^{1/\alpha}
<
n/(j(t)-1)
\)
which implies that
\(
j(t)-1 < nt^{-1/\alpha} \le j(t)
\).
By the definition of the ceiling function, we have that
\(j(t) = \mathtt{ceil}(nt^{-1/\alpha})\).
Therefore,
\begin{align*}
 & \mathtt{cdf}[\mathtt{esd}(n^{\alpha} \bSigma)](t)
  \\
  &
  =
\tfrac{1}{p}(
  p- \mathtt{ceil}(nt^{-1/\alpha}) +1)
  \\
  &
  =
    1-
\tfrac{\gamma}{n}
    \mathtt{ceil}(nt^{-1/\alpha})
    +
\tfrac{\gamma}{n}.
\end{align*}
Taking limit of both side as \(n\to \infty\) and using the fact that \(\lim_{n\to\infty} \mathtt{ceil}(nc)/n = c\) for any positive number \(c >0\), we get the desired result.
\end{proof}

\section{Proof of the Positivity Condition portion of \Cref{theorem:MP}}\label{section:positivity}
This section will focus on the  proof of \Cref{theorem:MP}, in particular the Positivity Condition (\Cref{assumption:positivity-condition}) portion.
Thus, throughout this section, we assume the setting of
\Cref{example:HDA}.
As mentioned in the main text, it is well-known that 
\Crefrange{assumption:self-consistent-kr}{assumption:MP-law} for the HDA model. As such, we assume these assumptions.
Now, using \Cref{assumption:MP-law} and elementary calculus, we first show that
\[\lim_{n\to \infty}
   \mathbb{E}\big[\tfrac{d}{dr}( r \mathcal{S}_{\mathtt{esd}(n^{\alpha} \check{\bG})}(-r))\big]
    =
    \left(\tfrac{dr}{dk}\right)^{-1}\cdot
    \tfrac{d}{dk}\left(
    k \mathcal{S}_{H}(-k)
  \right)
\]
  where \(r\) and \(k\) are as in \Cref{assumption:self-consistent-kr}.
  Thus, we reduce to showing the positivity of
\(\tfrac{dr}{dk}\) and
\(
    \tfrac{d}{dk}\left(
    k \mathcal{S}_{H}(-k)
  \right)
  \).
%   See
% \Cref{section:lower-bound-appendix}.
% \hfill \(\square\)

% \subsection{Continued from \Cref{section:positivity}}

Before proceeding, we recall several definitions and notations adapted from \citet{dobriban2018high}:
\begin{equation}
    \lim_{n\to \infty}
    \mathbb{E}
    \left[
\mathcal{S}_{\mathtt{esd}(n^{\alpha} \check{\bG})}(z)
\right]
    =
      v(z)
      \label{equation:companion-MP-law}
\end{equation}
is analogous to the \(v(z)\) defined in the paragraph immediately following \citep[Eqn.~(2)]{dobriban2018high}.
The difference is our
      \Cref{equation:companion-MP-law} is for the limit of the \(n^{\alpha}\)-scaled matrices \(n^{\alpha}  \check{\bG}\), rather than for \(\check{\bG}\) as in \citet{dobriban2018high}.

  Let \(H =
\lim_{n\to\infty}\mathtt{cdf}[ \mathtt{esd}(n^{\alpha} \bSigma) ]
  \)
  be the limiting distribution as in \Cref{assumption:EVD}.
Plugging in \(z = -r\) into  \citet[Eqn.~(A.1)]{dobriban2018high}, we have
\[
-\frac{1}{v(-r)} = -r - \frac{1}{\gamma} \int \frac{t dH(t)}{1 + tv(-r)}.
\]
Letting \(k \equiv k(r) := \frac{1}{v(-r)}\), we can rewrite the above as
\begin{equation}
  1= \frac{r}{k} +
  \frac{1}{\gamma} \int \frac{t dH(t)}{k + t}.
\label{equation:rk-equation}
\end{equation}
By construction, we have
\[
  \frac{1}{\gamma} \int \frac{t dH(t)}{k + t}
  =
    \lim_{n\to\infty} \frac{1}{n} \sum_{i=1}^{p}
    \frac{
      1
    }{1 + kn^{-\alpha}\lambda_{i}^{-1}}
\]
where the RHS is as in \Cref{assumption:self-consistent-kr}. Consequently, the tuple \(r,k\) from \Cref{assumption:self-consistent-kr} coincide with the earlier definition of \(k := \frac{1}{v(-r)}\) right before
\Cref{equation:rk-equation}. Having established the above, we now proceed to:

\begin{lemma}\label{lemma:positivity}
  Under the HDA model (\Cref{example:HDA})
  and the EVD condition (\Cref{assumption:EVD}), we have that
  \(    \lim_{n\to\infty}
\mathbb{E}
\big[
\tfrac{d}{dr}( r \mathcal{S}_{\mathtt{esd}(n^{\alpha} \check{G})}(-r))
\big]
>0
\).
\end{lemma}

\begin{proof}[Proof of \Cref{lemma:positivity}]
  By the product rule, we have
  \begin{equation*}
   \tfrac{d}{dr}\left( r \mathcal{S}_{\mathtt{esd}(n^{\alpha} \check{\bG})}(-r)\right)
    =
\mathcal{S}_{\mathtt{esd}(n^{\alpha} \check{\bG})}(r)
      -
   r \mathcal{S}_{\mathtt{esd}(n^{\alpha} \check{\bG})}'(-r)
  \end{equation*}
Now, taking the limit of the above equation on both side, we have
  \begin{align*}
    &
    \lim_{n\to \infty}
   \mathbb{E}\left[\tfrac{d}{dr}\left( r \mathcal{S}_{\mathtt{esd}(n^{\alpha} \check{\bG})}(-r)\right)\right
   ]
      \\
    &=
    \lim_{n\to \infty}
    \mathbb{E}
    \left[
\mathcal{S}_{\mathtt{esd}(n^{\alpha} \check{\bG})}(-r)
      -
   r \mathcal{S}_{\mathtt{esd}(n^{\alpha} \check{\bG})}'(-r)
\right]
    \\
    &=
      v(-r)
      -
      rv'(-r)
      \qquad \because
      \mbox{Definition of \(v\) and \(v'\)}
      \\&=
    \tfrac{d}{dr}\left(
    r v(-r)
    \right)
    \qquad \because \mbox{Product rule}
    \\&=
    \tfrac{d}{dr}\left(
    k \mathcal{S}_{H}(-k)
    \right)
    \qquad \because \mbox{
Marchenko-Pastur law (\Cref{assumption:MP-law})}
    \\&=
    \tfrac{dk}{dr}\cdot
    \tfrac{d}{dk}\left(k \mathcal{S}_{H}(-k) \right)
    \qquad \because \mbox{Chain rule}
    \\&=
\left(\tfrac{dr}{dk}\right)^{-1}\cdot\tfrac{d}{dk}\left(
    k \mathcal{S}_{H}(-k)
    \right)
    \qquad \because\mbox{Inverse function theorem}
  \end{align*}
  To complete the proof, it suffices to show that
  both
  \(    \frac{dr}{dk}
\)
and
  \(    \frac{d}{dk}\left(
    k \mathcal{S}_{H}(-k)
    \right)
\)
are positive which will be checked in the next two lemmas.
\end{proof}

\begin{lemma}\label{lemma:positive-of-drdk}
  The function
  \(\frac{dr}{dk}\) evaluated at \(k\) is positive.
\end{lemma}

\begin{proof}[Proof of \Cref{lemma:positive-of-drdk}]
 %  From
 % \citet[\S4]{silverstein1995analysis}, we know that
 % the function
 % \((-\infty,0) \ni r \mapsto v(r)\) is increasing.
 Recall that  \(k = \frac{1}{v(-r)}\).
 Thus, we have
 \[
   \frac{dk}{dr}(r) =
   (-1)
   \frac{1}{v(-r)^{2}} (-1)\cdot v'(-r)
   =
   \frac{v'(-r)}{v(-r)^{2}}.
 \]
 From the proof of \citet[Theorem 4.1]{silverstein1995analysis}, we see that \(v'(\cdot) > 0\) for all negative inputs.
 In particular, \(v'(-r)>0\) which implies that
 \(\frac{dk}{dr}\) is positive. By the inverse function theorem, we have
 \(\frac{dr}{dk} = (\frac{dk}{dr})^{-1}\) is also positive.
\end{proof}
\begin{lemma}\label{lemma:positivity-of-dkSk}
 The quantity
\(\frac{d}{dk}\left(k \mathcal{S}_{H}(-k) \right)\)
is positive.
\end{lemma}
\begin{proof}[Proof of \Cref{lemma:positivity-of-dkSk}]

Plugging in \(z = -r\) into   \citet[Eqn.~(3)]{dobriban2018high}, we have
\begin{equation}
v(-r) - \frac{1}{r} = \frac{1}{\gamma}\left(m(-r) - \frac{1}{r}\right).
\label{equation:v-to-m-transform}
\end{equation}
Now,
\begin{align}
  r m(-r)
  &= \gamma r v(-r) + (1-\gamma) \quad \because \mbox{\Cref{equation:v-to-m-transform}}\\
  &= \gamma \frac{r}{k}+ (1-\gamma) \quad \because \mbox{Definition of \(k\)}\\
  &= \left(\gamma - \int \frac{t dH(t)}{k + t}\right)    + (1-\gamma) \quad \because \mbox{\Cref{equation:rk-equation}}\\
  &= 1 - \int \frac{t dH(t)}{k + t}
    \\
  &=
    \int \frac{ kdH(t)}{k + t}
    \quad \because
    1 = \int dH(t) = \int \frac{k+t}{k+t} dH(t)\\
  &=
    k \mathcal{S}_{H}(-k).
\end{align}
Thus, differentiating under the integral, we have
\[
\frac{d}{dk}
(k \mathcal{S}_{H}(-k))
=
    \int \frac{d}{dk}\left( \frac{ k}{k + t}\right)dH(t)
=
\int \frac{t dH(t)}{(k+t)^{2}}
>0
\]
as desired.
\end{proof}

\section{Proof of \Cref{theorem:polynomial-lower-bound-specialized}  --- rapid norm growth}
\begin{proof}[Proof of \Cref{theorem:polynomial-lower-bound-specialized}]
  Let \(\tau' := \frac{\tau+\sigma^{2}}{2}\). Then by definition, we have \(\tau' \in (0,\sigma^{2})\).
  From \Cref{theorem:trade-off},
  we can pick \(r' >0\) so that the sequence of regularizers \(\varrho_{n}' = r' n^{-\alpha}\)
  satisfies
  \(\Etrain^{*}(r') = \tau'\).
  Next, we note that there exists \(c >0\) such that
  \[
    \Pr(
\Etrain(\hat{\bbeta}_{\varrho_{n}'})
\ge
\Etrain(\hat{\bbeta}_{\varrho_{n}})
)
\ge c
  \]
  for all \(n\) sufficiently large.
  Such a \(c\) is guaranteed to exist by the fact that
  \[
    \lim_{n}
\Etrain(\hat{\bbeta}_{\varrho_{n}'})
=
\tau'
=\frac{\tau+\sigma^{2}}{2}
>
\tau =
    \lim_{n}
\Etrain(\hat{\bbeta}_{\varrho_{n}}).
  \]
  Now, note that
  \begin{align}
    &
\Etrain(\hat{\bbeta}_{\varrho_{n}'})
\ge
\Etrain(\hat{\bbeta}_{\varrho_{n}})
    \\
    & \iff
\mbox{
  \(\hat{\bbeta}_{\varrho_{n}}\) is a
  \(
\Etrain(\hat{\bbeta}_{\varrho_{n}'})
  \)-near interpolator
}
    \\
    & \implies
      \|\hat{\bbeta}_{\varrho_{n}}\|^{2}
      \ge
\|\hat{\bbeta}_{\varrho_{n}'}\|^{2}
\end{align}
  From this, we conclude that there exists a \(c>0\) such that
\[
    \Pr(
      \|\hat{\bbeta}_{\varrho_{n}}\|^{2}
      \ge
\|\hat{\bbeta}_{\varrho_{n}'}\|^{2}
)
\ge c
  \]
  for all \(n\) sufficiently large.
  From this, it follows that
  \(    \mathbb{E}[\|\hat{\bbeta}_{\varrho_{n}}\|^{2}]
    \ge c
    \mathbb{E}[\|\hat{\bbeta}_{\varrho_{n}'}\|^{2}]
\).
By \Cref{theorem:polynomial-lower-bound} and the preceding inequality, we are done.
\end{proof}

\section{Code for implementation \(\mathcal{I}\) and \(\mathcal{J}\)}\label{section:code-appendix}
Implementation of the \(\mathcal{I}\) and \(\mathcal{J}\) functions from \Cref{definition:IJ-functions} can be implemented in SCIPY as:
\begin{lstlisting}[language=Python]
import scipy.special as sc
gamma = 0.5
alpha = 1.75

# I helper
I_gen = lambda x,k, alpha: x*sc.hyp2f1(1,(1/alpha), 1 + (1/alpha), -k*x**alpha)
# J helper
J_gen = lambda x,k, alpha: x*sc.hyp2f1(2,(1/alpha), 1 + (1/alpha), -k*x**alpha)

I = lambda k : I_gen(1/gamma, k, alpha) #\mathcal{I}
J = lambda k : J_gen(1/gamma, k, alpha) #\mathcal{J}

N = lambda k : 1 - I(k) # helper
D = lambda k : 1 - J(k) # helper

Etst = lambda k : 1/D(k) #\Etest/\sigma^2
Etrn = lambda k : N(k)**2/D(k) #\Etrain/\sigma^2
R  = lambda k : k*(1-I(k)) # \mathcal{R}
\end{lstlisting}

\section{Experiments}\label{sec:real-world-experiment}
Code for reproducing all figures are included in the official GitHub repository:

\url{https://github.com/YutongWangUMich/Near-Interpolators-Figures/}

For downloading the UCI regression datasets, we use the following repository:

\url{https://github.com/treforevans/uci_datasets}

For the neural tangent kernel, we use the official repository associated to \citet{arora2019harnessing}:

\url{https://github.com/LeoYu/neural-tangent-kernel-UCI}

In \Cref{figure:stock}-left, note that the curve corresponding to \texttt{stock.2-1} has the fastest spectra decay and simultaneously the worst trade-off.
Evidently, larger decay exponent corresponds to a poorer trade-off, especially for near-interpolators, i.e., as the training error approaches \(0\). This is in agreement with our theoretical results under random matrix theory assumptions illustrated in \Cref{figure:steepness}.

\begin{figure*}[h]
  \centering
\includegraphics[width = 0.44\textwidth]{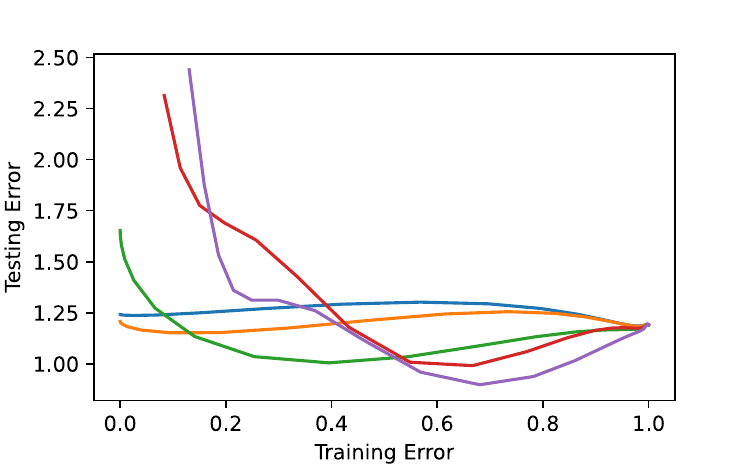}
~
~
~
\includegraphics[width = 0.42\textwidth]{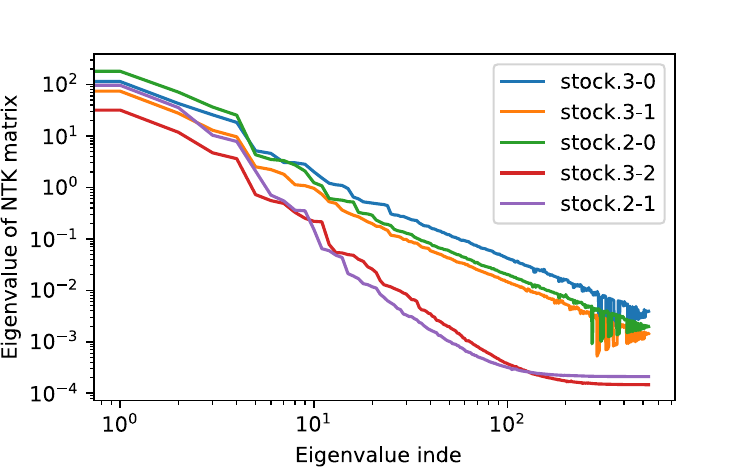}
\caption{\emph{Left}. Training/testing error trade-off on the ``\texttt{stock}'' dataset from the UCI regression dataset collection using kernel ridge regression with the neural tangent kernel.
Each curve is labeled by ``\texttt{DatasetName.d-f}''
where ``\texttt{d}'' and ``\texttt{f}'' represents the number
of layers and the number of \emph{fixed} layers 
in the NTK corresponding to ReLU networks.
\emph{Right}. The eigenvalue index vs eigenvalue plot of the NTK matrix exhibits power-law spectra. A tiny value is added to the eigenvalues for better visualization on the log-scale.
}\label{figure:stock}
\end{figure*}

\end{document}